\documentclass{article} %
\usepackage{iclr2027_conference,times}

\usepackage{amsmath,amsfonts,bm}

\def\eqref#1{equation~\ref{#1}}

\def\1{\bm{1}}

\def\vzero{{\bm{0}}}
\def\vone{{\bm{1}}}

\def\va{{\bm{a}}}
\def\vb{{\bm{b}}}

\def\vg{{\bm{g}}}

\def\vk{{\bm{k}}}
\def\vl{{\bm{l}}}
\def\vm{{\bm{m}}}
\def\vn{{\bm{n}}}

\def\vx{{\bm{x}}}
\def\vy{{\bm{y}}}

\def\mA{{\bm{A}}}
\def\mB{{\bm{B}}}

\def\mH{{\bm{H}}}
\def\mI{{\bm{I}}}

\def\mK{{\bm{K}}}
\def\mL{{\bm{L}}}
\def\mM{{\bm{M}}}
\def\mN{{\bm{N}}}

\def\mP{{\bm{P}}}

\def\mT{{\bm{T}}}
\def\mU{{\bm{U}}}
\def\mV{{\bm{V}}}
\def\mW{{\bm{W}}}
\def\mX{{\bm{X}}}
\def\mY{{\bm{Y}}}
\def\mZ{{\bm{Z}}}

\def\mSigma{{\bm{\Sigma}}}

\DeclareMathAlphabet{\mathsfit}{\encodingdefault}{\sfdefault}{m}{sl}
\SetMathAlphabet{\mathsfit}{bold}{\encodingdefault}{\sfdefault}{bx}{n}

\newcommand{\R}{\mathbb{R}}

\DeclareMathOperator*{\argmax}{arg\,max}
\DeclareMathOperator*{\argmin}{arg\,min}

\DeclareMathOperator{\Tr}{Tr}
\DeclareMathOperator{\St}{St}
\DeclareMathOperator{\polar}{polar}
\DeclareMathOperator{\diag}{diag}
\DeclareMathOperator{\cka}{CKA}
\DeclareMathOperator{\hsic}{HSIC}
\DeclareMathOperator{\hungarian}{hungarian\_matching}
\DeclareMathOperator{\hungarianbatch}{hungarian\_matching\_batch}
\DeclareMathOperator{\GeometricInit}{GeometricInit}

\usepackage{amsmath}

\usepackage{amsthm}
\newtheorem{lemma}{Lemma}
\usepackage{bbm}

\newcommand*{\inparagraph}[1]{\noindent\textbf{#1}\hspace{0.5em}}

\usepackage{tikz}
\usetikzlibrary{positioning, calc, patterns, shapes.arrows, shapes.geometric, fit, arrows, backgrounds, shapes.callouts}
\usepackage{pgfplots}
\usepgfplotslibrary{groupplots,dateplot}
\pgfplotsset{compat=newest}

\definecolor{TUMBlue}{HTML}{0065BD}
\definecolor{TUMSecondaryBlue}{HTML}{005293}
\definecolor{TUMSecondaryBlue2}{HTML}{003359}
\definecolor{TUMBlack}{HTML}{000000}
\definecolor{TUMWhite}{HTML}{FFFFFF}
\definecolor{TUMDarkGray}{HTML}{333333}
\definecolor{TUMGray}{HTML}{808080}
\definecolor{TUMLightGray}{HTML}{CCCCC6}
\definecolor{TUMAccentGray}{HTML}{DAD7CB}
\definecolor{TUMAccentOrange}{HTML}{E37222}
\definecolor{TUMAccentGreen}{HTML}{A2AD00}
\definecolor{TUMAccentLightBlue}{HTML}{98C6EA}
\definecolor{TUMAccentBlue}{HTML}{64A0C8}

\definecolor{TUMYellow}{RGB}{254, 215, 2}
\definecolor{TUMPink}{RGB}{181, 92, 165}
\definecolor{TUMPurple}{RGB}{143, 129, 234}
\definecolor{TUMRed}{RGB}{234, 114, 55}
\definecolor{TUMTeal}{HTML}{00857C}
\definecolor{TUMBrown}{HTML}{8C6240}

\usepackage{algorithm}
\usepackage[noend]{algpseudocode}

\algrenewcommand\alglinenumber[1]{\tiny #1}

\usepackage{multirow}
\usepackage{pdflscape}
\usepackage{array}
\usepackage{setspace}
\usepackage{tabularx}
\usepackage{siunitx}
\input{figures/experiments/tables.tex}

\newcolumntype{L}[1]{>{\raggedright\let\newline\\\arraybackslash\hspace{0pt}}m{#1}}
\newcolumntype{C}[1]{>{\centering\let\newline\\\arraybackslash\hspace{0pt}}m{#1}}
\newcolumntype{R}[1]{>{\raggedleft\let\newline\\\arraybackslash\hspace{0pt}}m{#1}}

\usepackage[most]{tcolorbox}
\tcbset{shield externalize}

\newcommand{\callout}[1]{%
    \par\noindent
    \begin{tcolorbox}[colframe=white!30!black, colback=white!98!black, width=\columnwidth, left=.7mm, right=.7mm, top=1mm, bottom=1mm, boxrule=.5mm, before skip=8pt, after skip=8pt]
        #1
    \end{tcolorbox}
    \par
}

\newcounter{finding}
\newcommand{\finding}[1]{%
    \stepcounter{finding}
    \callout{\textbf{Finding~\thefinding:}\hspace{.5em}#1}
}

\usepackage{caption}
\usepackage{placeins}
\usepackage{xurl}
\usepackage[protrusion=false,expansion=true,stretch=30,shrink=30]{microtype}
\AtBeginDocument{\microtypesetup{expansion=false}}
\usepackage{subcaption}
\usepackage{xspace}

\makeatletter
\DeclareRobustCommand\onedot{\futurelet\@let@token\@onedot}
\def\@onedot{\ifx\@let@token.\else.\null\fi\xspace}

\def\ie{\emph{i.e}\onedot} 
\def\cf{\emph{cf}\onedot}

\makeatother

\usepackage{enumitem}

\AtEndPreamble{
    \usepackage[capitalize]{cleveref}
    \crefname{section}{Sec.}{Secs.}
    \Crefname{section}{Section}{Sections}
    \Crefname{table}{Table}{Tables}
    \crefname{table}{Tab.}{Tabs.}
    \Crefname{appendix}{Appendix}{Appendices}
    \crefname{appendix}{Appx.}{Appxs.}
}

\usepackage[normalem]{ulem}

\usepackage[colorlinks=true, linkcolor=TUMSecondaryBlue2, citecolor=TUMSecondaryBlue2, urlcolor=TUMBlue]{hyperref}
\usepackage{url}
\title{
Shared Geometry As A Rosetta Stone:\\
Cross-Modal Alignment Without Paired Data
}

\author{Dominik Schnaus$^\text{\normalfont 1,2}$, Thomas Dag{\`e}s$^\text{\normalfont 1,2}$, Daniel Cremers$^{\text{\normalfont 1,2}\dagger}$, Xi Wang$^{\text{\normalfont 1,2,3,4}\dagger}$, Phillip Isola$^{\text{\normalfont 5}\dagger}$ \\[2mm]
$^\text{1}$TU Munich \quad $^\text{2}$MCML \quad $^\text{3}$Ulm University \quad $^\text{4}$ETH Zurich \quad $^\text{5}$MIT \quad $^\dagger$Equal advising\\[1mm]
{\small \textbf{Project page:} {\href{https://dominik-schnaus.github.io/unpaired-rosetta/}{dominik-schnaus.github.io/unpaired-rosetta}}}
}

\iclrpreprintcopy
\begin{document}

\maketitle

\begin{abstract}
Multimodal representations enable 
zero-shot classification and retrieval, but aligning independently trained models 
usually
requires large amounts of paired data.
Yet,
the Platonic Representation Hypothesis suggests that models trained on different modalities may 
converge spontaneously toward a shared representation geometry.
But then,
do we even need paired examples 
for
cross-modal alignment?
Remarkably, we show that paired examples are unnecessary for coarse cross-modal alignment.
Our simple Wasserstein Procrustes method
with a coarse geometric initialization (GeoWarP)
aligns two disjoint embedding sets by estimating a single orthogonal 
map
without 
seeing
any pairs.
Across
datasets, modalities,
and unimodal models,
we show 
that we can
consistently align
independently trained representations without pairs,
and standard geometric alignment metrics accurately predict when this is possible.
Nevertheless, 
we can
naturally benefit from 
paired examples.
In the very few-pair regime,
our method substantially outperforms existing 
ones,
while 
staying
competitive with pair-based methods 
with
more 
added
examples. 
Finally, we demonstrate that the 
resulting alignments can enable text-to-image generation without paired 
examples.
These results show that independently trained models often share enough geometry to establish cross-modal correspondence with little or no paired data.
\end{abstract}

\section{Introduction}

The discovery of the Rosetta Stone was a philology breakthrough. By having a 
Greek translation to a small text, the stone provided sufficient information to Champollion and Young to decipher entire ancient Egyptian scripts. Modern multimodal learning follows essentially the same principle.
Models 
across
modalities are 
usually
connected 
by
paired observations 
via
image captions~\citep{Radford:2021:LTV, jia2021scaling, zhai2023sigmoid}, 
biological~\citep{xiong2023scclip},
or 
scientific data~\citep{chang2024bidirectional}.
Example pairs
act as a multimodal Rosetta Stone 
illustrating
how one modality maps 
to
another.
This assumption underlies 
major
paradigms:
contrastive
learning
from image-caption pairs~\citep{Radford:2021:LTV, jia2021scaling, zhai2023sigmoid}
or
multimodal foundation models 
trained
on paired data 
at massive scale~\citep{bai2025qwen3}.
Even methods designed to reduce 
supervision still 
need some
known 
pairings
as 
cross-modal
anchors~\citep{Norelli:2023:ASI, Maiorca:2023:LST, Maniparambil:2024:DVL, yacobi2026learning, groger2026limited, roschmann2026sotalign}.
The prevailing view is that meaningful alignment 
requires
at least some 
observed correspondences.

\begin{figure}[t]
    \centering
    \includegraphics{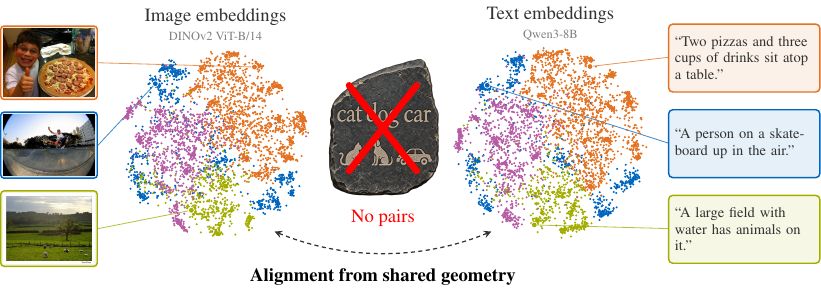} {} 
    \vspace{-1em}
    \caption{\textbf{Shared geometry enables unpaired alignment.} 
    Across
    modalities,
    embedding spaces of independent models
    share
    a similar 
    geometry.
    Prior works
    require a multimodal ``Rosetta Stone'' of paired 
    cross-modal
    examples to 
    align both
    spaces.
    In contrast, our Wasserstein Procrustes approach directly exploits this shared geometry 
    for aligning
    without observing a single 
    paired example.
    We here show
    the nearest neighbor of each image in the text embedding space
    of disjoint datasets.
    }
    \label{fig:teaser}
    \vspace{-0.5em}
\end{figure}

Recent findings 
fragilize this view.
Independently trained models often 
show
oddly
similar 
geometry.
The Platonic Representation Hypothesis (PRH)~\citep{huh2024platonic} proposes that 
models 
spontaneously
converge 
to
a shared representation as model and dataset scale increase.
While the extent of 
convergence 
is
unclear~\citep{groger2026revisiting, koepke2026back}, 
there is
substantial
representation agreement
across modalities, 
mostly coarsely~\citep{koepke2026back}.
This extends to
video~\citep{zhu2025dynamic}, audio~\citep{ngo2024language}, neuroscience~\citep{marcos2026platonic}, and scientific 
models~\citep{edamadaka2025universally, li2026platonic}.
These 
findings raise
a fundamental question:

\callout{\textbf{Research Question:}\hspace{.5em}Can two independently trained modalities be aligned from their embedding geometry alone, without observing a single corresponding pair?}

This question is challenging 
as both
embedding spaces 
have no shared coordinates, no known 
sample
correspondences,
and potentially different numbers
of samples and dimensions.
Prior works~\citep{Norelli:2023:ASI, Maiorca:2023:LST, Maniparambil:2024:DVL, yacobi2026learning, groger2026limited, roschmann2026sotalign} have shown that a 
few
paired examples 
may suffice
if
the
embedding geometry
is exploited.
Early evidence from
fully unpaired methods
~\citep{hoshen2018unsupervised, schnaus2025it} 
hints
that geometric structure 
may
recover cross-modal correspondences, but 
these methods
have been restricted to small-scale datasets or class-level embeddings.
Aligning disjoint sets when removing all paired examples remained an
unsolved problem.

In this paper, we show that such a blind cross-modal alignment is possible, at least at a coarse level.
By introducing a simple alignment procedure based on Wasserstein Procrustes~\citep{rangarajan1997softassign, cohen1999earth}, we jointly infer a correspondence between samples and an orthogonal mapping between embedding spaces.
Due to non-convex~\citep{alvarez2019towards} optimization,
random initialization is insufficient.
Inspired by mini-vec2vec~\citep{dar2025mini}, we instead find a coarse correspondence by independently matching the geometric structure of the two embedding spaces and use it to initialize the joint alignment.
We call the resulting procedure GeoWarP, for \textbf{Geo}metric initialization and \textbf{Wa}sse\textbf{r}stein \textbf{P}rocrustes, and it requires only the two sets of unimodal embeddings.
It uses no paired examples, class labels, or shared reference points.

Across 
many
datasets, 
vision models, and 
language models, we find that independently trained 
models can 
be aligned without paired data.
The recovered alignment captures coarse semantic structure
even when 
images and captions 
come
from different datasets.
In fact, our unpaired alignment quality is strongly predicted by 
geometric similarity 
between
embedding spaces,
measured by common scores
like
Centered Kernel Alignment (CKA)~\citep{kornblith2019similarity}.
This provides
an empirical criterion 
for 
unpaired alignment 
success chances.
This relationship
is not specific to image-text models and
extends 
to scientific and medical domains. 
Our
framework naturally generalizes to
few paired examples,
bridging the unpaired and (very) few-pair settings.
With very few pairs ($\le 100$), we substantially outperform baselines and remain competitive with more added pairs.
We 
also
demonstrate 
proof-of-concept
text-to-image generation without image-text pairs.
Together, our results suggest a different view of multimodal alignment:
paired examples are not always necessary 
for
cross-modal 
alignment.
Independent
models
often 
already organize their embedding spaces similarly enough
for
shared geometry
to
reveal coarse correspondences.
Unlike the historical one, 
a 
multi-modal Rosetta Stone 
might not be necessary,
although having a small one 
can
greatly refine the estimated alignment. 
Instead, 
leveraging shared
geometries provides the required coarse bridge.

\section{Related Work}
\label{sec:related}
Our 
work
builds on three observations.
First, independently trained embedding spaces within a modality 
can often be aligned
without 
pairs
via
almost
isometric~\citep{mikolov2013exploiting} geometry.
Second, independently trained models across modalities also exhibit shared 
geometry.
Third, 
shared geometry has already
reduced
paired-data 
needs
for cross-modal alignment.
We combine these ideas 
in
the zero-pair regime, showing that 
geometry alone enables
coarse cross-modal alignment.

\inparagraph{Unpaired alignment within a modality.} 
Unpaired alignment was 
first
studied in unsupervised 
translation, 
for aligning across languages independently trained word embedding spaces.
Two main approaches have emerged.
Adversarial methods~\citep{conneau2017word, zhang2017adversarial, zhang2017earth, xu2018unsupervised, mohiuddin2019revisiting} learn a transformation such that a discriminator cannot distinguish mapped source 
from target embeddings.
Optimization-based methods instead explicitly search for a correspondence between the samples.
This leads to Wasserstein Procrustes~\citep{artetxe2018robust, hoshen2018non, grave2019unsupervised, aboagye2022quantized, even2024aligning}, which jointly optimizes the correspondence and an orthogonal transformation, or to Gromov-Wasserstein~\citep{Alvarez-Melis:2018:GWA, marchisio2022bilingual} formulations that seek correspondences preserving pairwise distances.
These ideas have since been extended beyond word embeddings to whole-text representations, including the adversarial vec2vec~\citep{jha2026harnessing} and optimization-based mini-vec2vec~\citep{dar2025mini} methods, as well as to applications in bioinformatics~\citep{demetci2022scot, baker2026scot+} and neuroscience~\citep{thual2022aligning, marcos2026platonic}.
While these methods demonstrate that shared geometry can enable unpaired alignment, they consider representations within the same modality.
We ask whether the same principle can bridge different modalities.

\inparagraph{Shared geometry across modalities.}
Independently trained models 
across
modalities have also been shown to converge toward similar representations.
At the unit level, ``Rosetta Neurons''~\citep{dravid2023rosetta} identified 
neurons 
encoding shared concepts 
in
independently trained vision models.
At the representation level, 
different modality
representations 
can be connected 
via
simple linear or orthogonal transformations~\citep{merullo2022linearly, koh2023grounding, Maiorca:2023:LST, gupta2026canonicalizing, li2024vision}.
Unpaired data from other modalities can also improve unimodal models~\citep{gupta2026better}.
Relative representations~\citep{Moschella:2023:RRE} go further by expressing samples through 
similarities to paired reference examples, using these shared relationships to place different modalities in a common space.
Such
reference examples form a multimodal ``Rosetta Stone''~\citep{Norelli:2023:ASI}.
The Platonic Representation Hypothesis~\citep{huh2024platonic} provides a broader framework for these observations, proposing that representations from different models converge toward a shared geometry as model and training data scale increase.
Across vision and language models, empirical studies consistently find shared geometric structure, with 
alignment increasing particularly in coarse neighborhood structure~\citep{groger2026revisiting, koepke2026back}.
Similar alignment has also been observed across audio-language~\citep{ngo2024language}, video-language~\citep{zhu2025dynamic}, and scientific models~\citep{edamadaka2025universally, li2026platonic}, and between image encoders and diffusion models~\citep{tasker2026universal}.
Prior work measures this shared structure on paired representations.
We use it to map between embedding spaces without paired data, showing a strong Platonic Representation Hypothesis~\citep{jha2026harnessing} across modalities.

\inparagraph{Cross-modal alignment with few pairs.}
This shared geometry has been increasingly 
used
to reduce the paired-data requirement for cross-modal alignment.
Early contrastive 
works~\citep{Radford:2021:LTV, jia2021scaling, zhai2023sigmoid} learn aligned representations directly from hundreds of millions to billions of paired examples.
However, much of this supervision can be replaced by strong pretrained unimodal representations.
LiT~\citep{zhai2022lit} uses a frozen vision encoder to reduce the amount of paired data substantially.
More explicitly, ASIF~\citep{Norelli:2023:ASI} directly uses relative representations~\citep{Moschella:2023:RRE} as a shared coordinate system between modalities and STRUCTURE~\citep{groger2026limited} regularizes the aligned space to preserve geometric structure.
Further methods use linear transformations~\citep{Maiorca:2023:LST}, CKA~\citep{Maniparambil:2024:DVL}, spectral representations~\citep{yacobi2026learning}, or optimal transport~\citep{roschmann2026sotalign} to reach the few-pair regime.
These methods demonstrate that paired data can primarily serve to establish a connection between embedding spaces whose internal geometry is already informative.

Rare
works have 
studied
the fully unpaired cross-modal setting, 
but
for
more restricted assumptions.
\citet{hoshen2018unsupervised} 
align supervised 
representations
with an 
adversarial CycleGAN~\citep{zhu2017unpaired},
and \citet{schnaus2025it} align class-aggregated 
embeddings using pairwise distances.
These
results 
suggest
cross-modal 
geometry
information
may
recover 
matchings
without 
examples,
but they are limited to small-scale settings or 
class-aggregated embeddings.
In contrast, we
tackle
disjoint 
embedding sets
from 
independent
models and use no paired examples, class labels, or shared 
reference points.
We show that 
shared geometry alone can recover a coarse cross-modal alignment.

\providecommand{\kmeans}{k\text{-means}}
\setlength{\fboxsep}{1pt}

\begin{figure*}[t]
\vspace{-1em}
\centering
\noindent\begin{minipage}[t]{0.52\textwidth}
    \centering
    \begin{algorithm}[H]
        \caption{GeoWarP}
        \label{alg:ours}
        {\scriptsize
        \begin{spacing}{1.10054}
        \begin{algorithmic}[1]
           \State \textbf{In:}$\mX\!\in\!\R^{n\!\times\!d_X}$, $\mY\!\in\!\R^{m\!\times\!d_Y}$, $C\!=\!30$, $S\!=\!30$, $b\!=\!10^4$, $\!R\!=\!100$
           \State \textbf{Optional:} $\hat{\mX} \in \R^{p \times d_X}$,
                  $\hat{\mY} \in \R^{p \times d_Y}$ pairs
           \State \textbf{Out:} $\mW \in \St(d_X, d_Y)$
           \State $\mT \gets
                  \GeometricInit(\mX, \mY, \hat{\mX}, \hat{\mY}, C, S, b)$
                  \hfill \cref{alg:init}
           \State \colorbox{TUMDarkGray!20}{$\mW \gets \polar\big(\mX^{\top} \mT \mY$}
                  \colorbox{TUMPink!20}{$+ \frac{\min(n,m)}{p}\hat{\mX}^{\top}\hat{\mY}$}
                  \colorbox{TUMDarkGray!20}{$\big)$} \label{line:procrustes1}
           \For{$r = 1, \dots, R$}
                \State $\mX_B, \mY_B \gets$
                       random batches of $b$ samples from $\mX$, $\mY$
                \State \colorbox{TUMAccentGreen!20}{$\mT \gets
                       \hungarian(
                       \mX_B\mW\mY_B^{\top}
                       )$}\label{line:hungarian1}
                \State \colorbox{TUMDarkGray!20}{$\mW \gets \polar\big(\mX_B^{\top}\mT\mY_B$}
                       \colorbox{TUMPink!20}{$+ \frac{b}{p}\hat{\mX}^{\top}\hat{\mY}$}
                       \colorbox{TUMDarkGray!20}{$\big)$}\label{line:procrustes2}
           \EndFor
           \State \textbf{return} $\mW$
        \end{algorithmic}
        \end{spacing}
        }
    \end{algorithm}
\end{minipage}%
\hfill
\noindent\begin{minipage}[t]{0.46\textwidth}
    \centering
    \begin{algorithm}[H]
        \caption{Geometric initialization}
        \label{alg:init}
        {\scriptsize
        \begin{algorithmic}[1]
              \State \textbf{In:} $\mX\!\in\!\R^{n\!\times\!d_X}$, $\mY\!\in\!\R^{m\!\times\!d_Y}$, $C\!=\!30$, $S\!=\!30$, $b\!=\!10^4$
              \State \textbf{Optional:} $\hat{\mX} \in \R^{p \times d_X}$,
                     $\hat{\mY} \in \R^{p \times d_Y}$ pairs
              \State \textbf{Out:}
                     $\mT \in \R^{n \times m}$
                     initial transport plan
              \For{$s = 1, \dots, S$}
                     \State $\mX_B, \mY_B \gets$
                            random batches of $b$ samples from $\mX$, $\mY$
                     \State \colorbox{TUMAccentOrange!20}{$\mA_s \gets
                            \kmeans(\mX_B, C)$;
                            $\mB_s \gets
                            \kmeans(\mY_B, C)$} \label{line:kmeans}
                     \State \colorbox{TUMBlue!20}{$\mP_s \gets
                            \displaystyle\argmax_{\mP \in \mathcal{P}_C}
                            \cka\left(
                            \begin{pmatrix}
                                   \mA_s\\
                                   \colorbox{TUMYellow!20}{$\hat{\mX}$}
                            \end{pmatrix},
                            \begin{pmatrix}
                                   \mP\mB_s\\
                                   \colorbox{TUMYellow!20}{$\hat{\mY}$}
                            \end{pmatrix}
                            \right)$} \label{line:qap}
                     
              \EndFor
              \State $\mM \gets \frac{1}{SC}\sum_{s=1}^S\mA_s^{\top}\mP_s\mB_s$
              \State \colorbox{TUMTeal!20}{$\mT \gets \hungarianbatch(\mX \mM \mY^{\top}, b)$}\label{line:hungarian2}
              \State \textbf{return}
                     $\mT$
        \end{algorithmic}
        }
    \end{algorithm}
\end{minipage}%
\caption{
\textbf{
GeoWarP alignment.}
We first 
get
a coarse correspondence by repeatedly 
\colorbox{TUMAccentOrange!20}{clustering} and 
\colorbox{TUMBlue!20}{matching} the geometric structure of both embedding spaces (right). 
From this initialization, we alternate 
\colorbox{TUMAccentGreen!20}{correspondence estimation}
and 
\colorbox{TUMDarkGray!20}{orthogonal Procrustes updates}
to refine the alignment (left). 
Paired examples, when available, enter both stages naturally as a 
\colorbox{TUMPink!20}{linear}\colorbox{TUMYellow!20}{ term}.
}
\label{fig:algorithms}
\vspace{-1em}
\end{figure*}

\setlength{\fboxsep}{1pt}

\section{Wasserstein Procrustes with Geometric Initialization}
Our goal is to see if shared geometry alone is sufficient to align independently trained embedding spaces across modalities.
We deliberately use a simple alignment procedure that we call GeoWarP.
It is based on Wasserstein Procrustes,
summarized in \cref{alg:ours},
and inspired by
prior work~\citep{artetxe2018robust, grave2019unsupervised, aboagye2022quantized, even2024aligning}.
Let $(\vx_i)_{i=1}^n\!=\!\mX\!\in\!\R^{n\times d_X}$, $(\vy_j)_{j=1}^m\!=\!\mY\!\in\!\R^{m\times d_Y}$, $d_X \leq d_Y$, be centered, normalized embeddings from independently trained models.
Their rows and feature coordinates have no known correspondence.

\inparagraph{Wasserstein Procrustes.}
We jointly estimate a map $\mW$ and a transport plan $\mT \in \Pi(\va,\vb)$:
\begin{equation}
\label{eq:wasserstein_procrustes}
\mW^{\text{WP}}, \mT^{\text{WP}} \in
\!\!\!\!\!\!\!
\argmin_{\mW \in \St(d_X, d_Y),\; \mT \in \Pi(\va, \vb)}\; \sum_{i=1}^n\sum_{j=1}^m \mT_{ij} \|\mW \vx_{i} - \vy_j\|_2^2
= \argmax \Tr(\mX\mW\mY^{\top}\mT^{\top}),
\end{equation}
with $\mW\in\St(d_X,d_Y)$ a semi-orthogonal matrix.
Given a
correspondence, the optimal map is given in closed form by orthogonal Procrustes~\citep{schonemann1966generalized}.
Given a
map, the optimal correspondence is 
found
by linear assignment.
Alternating 
updates is simple, but the joint problem is non-convex~\citep{alvarez2019towards}, and random or uniform initializations 
yield
little 
matching
signal.
Common initializations from word translation assume a high degree of isometry.
However, the shared geometry across modalities is mainly coarse~(\citet{koepke2026back}, \cf \cref{app:alignment_granularity}).
This motivates initializing the correspondence between clusters rather than between individual samples.

\inparagraph{Coarse geometric initialization.}
We build on the cluster initialization of mini-vec2vec \citep{dar2025mini}.
For each of $S$ repetitions, we independently sample and cluster both spaces into $C$ centers $\mA_s$ and $\mB_s$ with $k$-means (\colorbox{TUMAccentOrange!20}{Line~\ref{line:kmeans}}). We match the centers by the permutation that maximizes their linear CKA, $\mP_s\in\argmax_{\mP\in\mathcal{P}_C}\cka(\mA_s,\mP\mB_s)$ (\colorbox{TUMBlue!20}{Line~\ref{line:qap}}).
This is a quadratic assignment problem (QAP; \citet{quadrianto2008kernelized, Maniparambil:2024:DVL, schnaus2025it}).
We replace 2-opt~\citep{croes1958_2opt} with multiple initializations from mini-vec2vec with MPOpt with its primal heuristic~\citep{hutschenreiter2021fusion}, since our ablations show that the quality of the QAP solution is the most important algorithmic choice (\cf \cref{app:ablation}).
Importantly,
initialization is related to Gromov-Wasserstein (GW) alignment for preserving 
intra
similarities:
\begin{equation}
\label{eq:gromov_wasserstein}
\mT^{\text{GW}} \in \argmax_{\mT \in \Pi(\va, \vb)}\;
\Tr(\mK_X\mT\mK_Y\mT^\top) \qquad \text{with }\; \mK_X=\mX\mX^\top,\ \mK_Y=\mY\mY^\top .
\end{equation}
Maximizing CKA 
on
permutations is 
its
QAP on centered kernels (\cref{app:algorithm}, \cref{lem:cka_qap}).
Moreover, 
matched clusterings are rank-$C$ GW transport plans with cross-covariance $\frac{1}{C}\mA_s^\top\mP_s\mB_s$.
We average the resulting plans to get $\mM=\frac{1}{SC}\sum_s\mA_s^\top\mP_s\mB_s$, which summarizes a low-rank approximation to 
the sample-level GW correspondence without materializing an $n\times m$ transport matrix.

\inparagraph{Read-out and refinement.}
Sample-level pseudo-pairs
are found by
one conditional-gradient step on \cref{eq:gromov_wasserstein} from the averaged low-rank plan,
reducing
to linear assignment 
for
scores $\mX\mM\mY^\top$.
Full assignment is prohibitive on large datasets, so we randomly 
block
partition 
them
and solve 
independent one-to-one assignments between 
pairs of corresponding blocks 
via
the Jonker-Volgenant algorithm~\citep{jonker1988shortest} (\colorbox{TUMTeal!20}{Line~\ref{line:hungarian2}}).
Combining 
block assignments 
gives
$\min(n,m)$ pseudo-pairs, 
requiring
only $b\times b$ score matrices,
initializing
$\mW=\polar(\mX^\top\mT\mY)$ by Procrustes, 
$\polar(\mN) = \mU\mV^{\top}$ for $\mN = \mU\boldsymbol{\Sigma}\mV^{\top}$ (\colorbox{TUMDarkGray!20}{Line~\ref{line:procrustes1}}).
Wasserstein Procrustes 
refinement
on batches
alternates
updates of matching via linear assignment (\colorbox{TUMAccentGreen!20}{Line~\ref{line:hungarian1}}) and mapping via Procrustes (\colorbox{TUMDarkGray!20}{Line~\ref{line:procrustes2}}).

\inparagraph{Few-pair extension.}
We can
naturally incorporate a 
few
paired examples.
Known pairs contribute to the cluster-matching objective during initialization as a linear term (\colorbox{TUMYellow!20}{Line~\ref{line:qap}}) and to the Procrustes 
update during refinement (\colorbox{TUMPink!20}{Line~\ref{line:procrustes1} and Line~\ref{line:procrustes2}}).
We weigh the paired examples in the Procrustes step such that their total contribution is comparable to that of the pseudo-pairs obtained from the unpaired data.
Thus, the same algorithm continuously interpolates between the fully unpaired and few-pair settings.
\Cref{app:unpaired_alignment_methods} provides derivations, complexity, sensitivity, and implementation details.
The ablations show that the quality of the QAP solution matters most and that our method is stable with
respect to all four of its hyperparameters, with larger values generally giving a better alignment.

\section{Experiments}
\label{sec:experiments}
Our experiments answer three questions:
\begin{enumerate}[noitemsep, topsep=0pt, leftmargin=*]
    \item \textbf{Can independently trained modalities be aligned without pairs?} 
    Across 7 vision, 3 language models, and 4 datasets, our method consistently recovers coarse cross-modal structure, including when the two modalities are drawn from different datasets.
    The quality of the alignment depends strongly on the shared geometry that can be measured with common metrics (\cref{subsec:unpaired}).
    \item \textbf{Does 
    our
    approach generalize 
    to less common modalities for alignment?
    }
    Our simple aligner remains competitive with domain-specific methods on established unpaired benchmarks (\cref{subsec:domain_specific}).
    It
    successfully aligns a broad range of previously unexplored modality pairs when their embedding geometries are sufficiently similar (\cref{subsec:other_domains}).
    \item \textbf{What is the additional benefit of 
    few pairs?
    } 
    Combining our geometric initialization with 
    very few
    pairs substantially improves sample efficiency and outperforms existing few-pair methods by large margins (\cref{subsec:few_pair}). 
    We visualize the resulting alignment through cross-modal retrieval, PCA projections, and text-to-image generation (\cref{subsec:text2img}).
\end{enumerate}

\inparagraph{Experimental setup.}
We measure cross-modal correspondence using Fraction of Samples Closer Than the True Match (FOSCTTM; \citet{liu2019jointly}). This is the average rank normalized to $[0, 1]$, where $0.5$ indicates random matching and $0$ perfect alignment, independent of batch size.
We also evaluate zero-shot classification by mapping image embeddings to the language space and selecting the closest class prompt.
We report results on 5 random seeds and report mean and standard deviation for every metric.
The aligner is fitted exclusively on the \emph{disjoint} unimodal embeddings, and the classification datasets are used only for evaluation, with no validation-based model selection.
The same hyperparameters are used for all experiments.

\subsection{Unpaired Cross-Modal Alignment Is Possible}
\label{subsec:unpaired}
We first investigate unpaired alignment of independently trained vision and languages models.

\inparagraph{Setup.}
We evaluate 7 vision 
and 3 language models 
on
MS COCO~\citep{chen2015microsoft} and 3 detailed captioning datasets: Stanford Paragraph Captions (SPC)~\citep{krause2016paragraphs}, DCI~\citep{urbanek2024picture}, and DOCCI~\citep{onoe2024docci}.
Baselines are
vec2vec~\citep{jha2026harnessing} and mini-vec2vec~\citep{dar2025mini}, 
which are
adversarial and optimization-based 
methods for
unpaired alignment.

\begin{figure}[t]
    \centering
    \captionsetup[subfigure]{skip=6pt}%
    \captionsetup{skip=8pt}
    \begin{subfigure}[t]{\textwidth}
        \centering
        \includegraphics{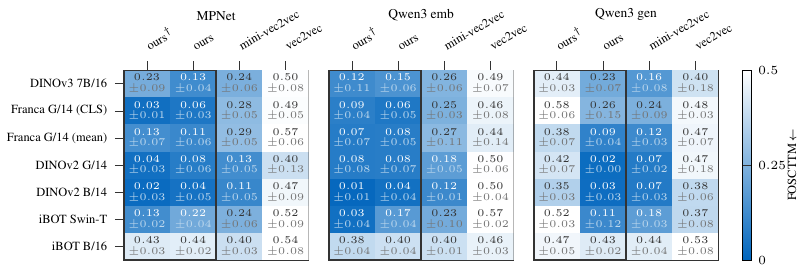} %
        \caption{Unpaired alignment performance. $^\dagger$Cross-dataset: images from MS COCO, captions from SPC.}
        \label{fig:unpaired_coco}
    \end{subfigure}

    \vspace{0.8em}
    \begin{subfigure}[b]{0.57\textwidth}
        \centering
        \includegraphics{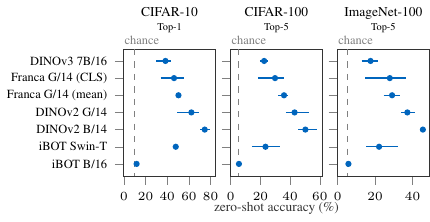} %
        \caption{Semantic transfer via zero-shot classification.}
        \label{fig:unpaired_class}
    \end{subfigure}%
    \hfill
    \begin{subfigure}[b]{0.41\textwidth}
        \centering
        \includegraphics{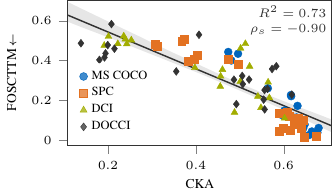} %
        \caption{Shared geometry predicts alignment.}
        \label{fig:alignment_metrics}
    \end{subfigure}
    \caption{
    \textbf{
    Shared-geometry for vision-language alignment without paired examples.}
    (\subref{fig:unpaired_coco})~FOSCTTM of unpaired aligners 
    on
    vision and 
    language models on MS COCO, 
    outperforming
    prior works on most model pairs.
    We can handle
    images and captions come from different datasets (ours†).
    (\subref{fig:unpaired_class})~Our
    alignment transfers semantics
    enabling
    zero-shot classification.
    (\subref{fig:alignment_metrics})~Points
    are
    vision-language-dataset combinations. The more similar 
    the
    embedding spaces,
    the better 
    our unpaired alignment.
    }
    \label{fig:unpaired_main}
\end{figure}

\inparagraph{Results.}
\cref{fig:unpaired_main} shows the results on MS COCO, 
those
on DCI, DOCCI, and SPC 
are
in \cref{app:unpaired}.
We 
see
in \cref{fig:unpaired_coco} that vec2vec is generally close to random matching, while mini-vec2vec already recovers substantial cross-modal structure.
Our method 
outperforms
mini-vec2vec on 17 of 21 vision-language model pairs, 
with
an average FOSCTTM of $0.154$ compared with $0.223$.
For reference, CLIP ViT-L/14, trained on 400M paired image-text examples, achieves $0.0006$.
Thus, our method does not recover the exact sample-wise correspondence, but it does recover substantial coarse cross-modal structure.

This coarse alignment is sufficient to transfer semantic information across modalities (\cf \cref{fig:unpaired_class}).
Without using any classification data during alignment, our method achieves average zero-shot accuracies of $47.4\%$ on CIFAR-10 (Top-1), $30.1\%$ on CIFAR-100 (Top-5), and $26.5\%$ on ImageNet-100 (Top-5). 
On the model combination from \citet{schnaus2025it}, which is optimized directly on CIFAR-10, our method nearly doubles the Top-1 accuracy from $37.3\%$ to $69.1\%$.
On the detailed captioning datasets, we observe the same overall behavior, although generative token pooling~\citep{wang2026the} performs substantially worse.
We investigate this difference in \cref{app:gen_token_pooling}.

So far, both modalities come from the same 
multimodal
dataset.
Yet
our method does not need this 
assumption.
To test this, we align 
MS COCO
images
with 
SPC
captions,
so that no ground-truth pairing exists between 
both
datasets.
Most model pairs remain similarly alignable (\cf \cref{fig:unpaired_coco}, ours$^\dagger$), with an average FOSCTTM of $0.236$ compared with $0.154$ when both modalities are drawn from MS COCO.
Of note,
the small degradation mainly comes from the combination of SPC and generative token pooling which is similar in the single dataset experiment (\cf \cref{fig:panel_StanfordParagraphCaptioning}).
This shows that our method does not rely on recovering a hidden pairing 
from
two views of the same dataset.

Regarding
unpaired alignment 
success,
all tested geometry 
scores,
including centered kernel alignment (CKA)~\citep{kornblith2019similarity} in \cref{fig:alignment_metrics} and others in \cref{fig:alignment_metrics_full}, strongly correlate with 
alignment quality 
given by
FOSCTTM.
Models 
with more similar neighborhood and distance structure are consistently easier to align,
while we naturally fail to align geometrically inconsistent models.

\subsection{Unpaired Alignment Generalizes Across Domains}
\label{subsec:domain_specific}
On vision-text, our method can align without paired data.
We 
now
study how
we 
align
non vision-text
problems where specialized methods have been designed for the target domain.

\inparagraph{Setup.}
We evaluate language alignment on NQ~\citep{kwiatkowski2019natural} against vec2vec~\citep{jha2026harnessing} and mini-vec2vec~\citep{dar2025mini}, RNA-to-ATAC alignment on PBMC~\citep{10xgenomics2021pbmc} against SCOT+~\citep{baker2026scot+}, and cross-subject fMRI alignment on NSD~\citep{allen2022massive} against Platonic Brain~\citep{marcos2026platonic}.
Unlike the Platonic Brain benchmark, we report average rather than best performance, since without supervision one cannot select the best-performing seed.
Additional per-pair results are given in \cref{app:domain_specific}.

\begin{figure}[t]
    \centering
    \begin{minipage}[c]{0.355\textwidth}
        \centering
        {\scriptsize\setlength{\tabcolsep}{3pt}\begin {tabular}{llS[table-format=1.4]@{\,\ensuremath {\pm }\,}S[table-format=1.4]}%
\toprule Benchmark&Method&\multicolumn {2}{c}{FOSCTTM $\downarrow $}\\\midrule %
NQ&vec2vec&0.4094&0.1628\\%
&mini-vec2vec&0.0003&0.0015\\%
&ours&\protect \bfseries  0.0000&\protect \bfseries  0.0000\\\midrule %
PBMC&SCOT+&0.1215&0.0275\\%
&ours&\protect \bfseries  0.0887&\protect \bfseries  0.0033\\\midrule %
NSD f\allowbreak {}MRI&platonic brain&0.0353&0.0413\\%
&ours&\protect \bfseries  0.0332&\protect \bfseries  0.0253\\\bottomrule %
\end {tabular}
}%
        \subcaption{Domain-specific benchmarks.}
        \label{tab:domain_main}
    \end{minipage}%
    \hfill
    \begin{minipage}[c]{0.64\textwidth}
        \centering
        \includegraphics{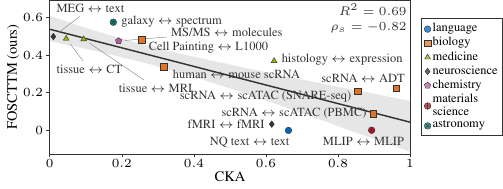} %

        \subcaption{Alignment against shared geometry, one point per modality pair.}
        \label{fig:other_domains}
    \end{minipage}

    \caption{\textbf{
    {Shared geometry predicts unpaired alignment across scientific domains}.}
    (\subref{tab:domain_main})~Our generic aligner is competitive with specialized methods on language (NQ, ten
    encoder pairs), biology (PBMC, RNA $\to$ ATAC) and neuroscience (NSD fMRI, fifty-six subject pairs). The
    non-aggregated metrics of each benchmark are reported in \cref{app:domain_specific}.
    (\subref{fig:other_domains})~Each point corresponds to a modality pair spanning seven scientific domains. 
    As in vision and language, modality pairs with more similar embedding geometry are consistently easier to align without paired data.
    }
    \label{fig:domain_specific}
\end{figure}

\inparagraph{Results.}
Our method is competitive with mini-vec2vec on NQ and achieves lower mean FOSCTTM than the specialized baselines on both PBMC 
($0.089$ compared to $0.121$ for SCOT+)
and NSD fMRI 
($0.033$ compared to $0.035$ for Platonic Brain), 
\cf \cref{tab:domain_main}.
Thus, 
our
simple geometry-based aligner transfers across language, biology, and neuroscience without task-specific adaptation.

\subsection{Shared Geometry Predicts Alignment Across Modalities}
\label{subsec:other_domains}
Here,
we
study
whether shared geometry can identify new settings in which unpaired alignment is likely to work, including settings without established alignment baselines.

\inparagraph{Setup.}
We collect independently trained models from medicine, biology, chemistry, materials science, astronomy, neuroscience, and language.
For each modality pair, we
compare CKA geometric similarity with our aligner's FOSCTTM.
Full modality pairs and results are provided in \cref{app:other_domains}.

\inparagraph{Results.}
\Cref{fig:other_domains} shows the same relationship observed in vision and language across these diverse settings. 
CKA explains $69\%$ of the variance in alignment performance ($R^2=0.69$), with a Spearman correlation of $\rho_s=-0.82$.
Models with more similar geometric structure are consistently easier to align without pairs, while poorly aligned spaces remain difficult to map.

\finding{%
For many models and domains,
independent 
embedding spaces 
of different modalities
share sufficient geometry for coarse alignment without pairs, even across different datasets.
The success of unpaired alignment depends strongly on the amount of shared geometry.
}

\subsection{Few-Pair Alignment}
\label{subsec:few_pair}
We showed that with shared geometry we can yield
reasonable
cross-modal correspondences without 
paired examples.
Yet having some examples can boost performance.
We next 
study
how our geometric initialization lowers the amount of paired examples to obtain a strong alignment.

\inparagraph{Setup.}
We use the same MS COCO setting with DINOv2 ViT-B/14 and Qwen3-8B gen and provide from 0 to 1,000 paired image-caption examples.
Baselines are
linear or orthogonal maps~\citep{Maiorca:2023:LST}, ASIF~\citep{Norelli:2023:ASI}, Local CKA~\citep{Maniparambil:2024:DVL}, SUE~\citep{yacobi2026learning}, STRUCTURE~\citep{groger2026limited}, and SOTAlign~\citep{roschmann2026sotalign}.

\begin{figure}[t]
    \centering
    \includegraphics{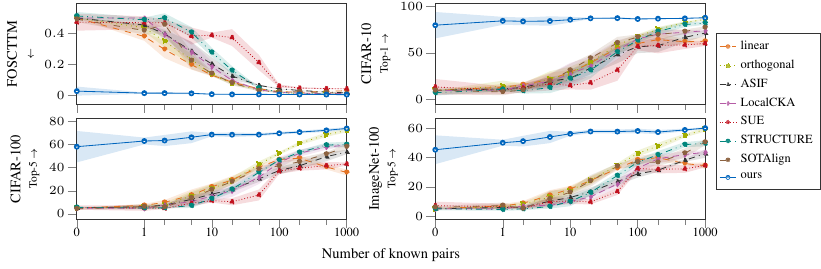} %
    \vspace{-1em}
    \caption{\textbf{Shared geometry drastically reduces the amount of paired supervision beneficial for improved alignment.} 
    Alignment quality (FOSCTTM) and downstream zero-shot classification accuracy are shown as a function of the number of known image-text pairs. 
    Starting from a strong unpaired solution (leftmost point), our geometry-based method consistently outperforms existing few-pair methods, with the largest gains below 100 pairs.
    }
    \label{fig:few_pair}
\end{figure}

\inparagraph{Results.}
\Cref{fig:few_pair} shows alignment quality 
versus
the number
of 
paired examples.
We
outperform 
all
baselines at 
all
budgets, with 
largest gains in the low-pair regime:
with up to $20$ pairs, 
our
FOSCTTM is $14\times$ to $28\times$ 
lower than 
the strongest baseline's.
The gap remains $5.4\times$ at $50$ pairs and $3.2\times$ at $100$,
while the orthogonal map nearly catches up only at $1{,}000$.
Remarkably, 
baselines 
matching
our 
zero-pair
FOSCTTM 
need $100$ to $500$ pairs to do so, 
and SUE never does.
The advantage is even stronger downstream: our zero-pair model reaches $79.6\%$ top-1 accuracy on CIFAR-10, a level only STRUCTURE and 
orthogonal map ever reach, and only with $500$ pairs. 
Conversely, 
we
only 
need
$20$ pairs 
to match or outperform five of 
seven baselines given $1{,}000$ pairs.

\finding{%
Leveraging shared geometry also yields a clear advantage in the few-pair regime.
}

\subsection{Visualizing the Alignment with Text-to-Image Generation}
\label{subsec:text2img}
We 
here
visualize 
our
coarse correspondence 
through
four different experiments: 
text-to-image retrieval, image-to-text retrieval, PCA projection of the aligned space, and text-to-image generation.

\inparagraph{Setup.}
We use DINOv2 ViT-B/14 and all-mpnet-base-v2 for all four experiments. 
For the first three experiments, we show the results when using 
MS COCO
images
and 
SPC
captions. 
Retrieval and PCA projections are done on the held-out MS COCO validation split. 
For text-to-image generation, we fit the aligner on disjoint MS COCO images and captions,
using
0 
to
100 paired examples. 
We then train a small RAE diffusion model~\citep{zheng2026diffusion} exclusively on ImageNet-1K 
images~\citep{russakovsky2015imagenet},
conditioned on 
mean-pooled DINOv2 embeddings.
At inference,
captions 
are
mapped 
to
the
image embedding space with our aligner,
which
condition the visual generator. 
Thus, paired data enters the pipeline only through the cross-modal aligner. 
The goal is not to build a competitive text-to-image system, but to 
view directly our
cross-modal alignment.

\begin{figure}[t]
    \centering
    \captionsetup[subfigure]{skip=4pt}%
    \begin{subfigure}[t]{0.48\textwidth}
        \centering
        \includegraphics{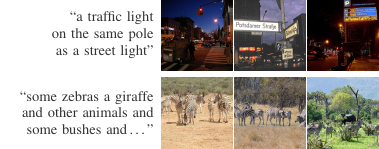} %
        \caption{Text-to-image retrieval.}
        \label{fig:qualitative_text_to_image}
    \end{subfigure}%
    \hfill
    \begin{subfigure}[t]{0.48\textwidth}
        \centering
        \includegraphics{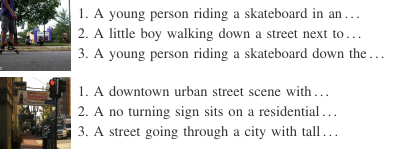} %
        \caption{Image-to-text retrieval.}
        \label{fig:qualitative_image_to_text}
    \end{subfigure}

    \vspace{0.4em}
    \begin{subfigure}[b]{0.48\textwidth}
        \centering
        \includegraphics{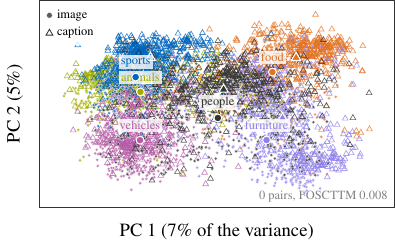} %
        \caption{Shared semantic regions after alignment.}
        \label{fig:qualitative_space}
    \end{subfigure}%
    \begin{subfigure}[b]{0.48\textwidth}
        \centering
        \includegraphics{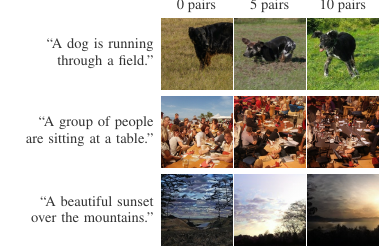} %
        \caption{Text-to-image generation.}
        \label{fig:text2img_small}
    \end{subfigure}

    \caption{\textbf{
    {Unpaired alignment recovers semantic structure}.} 
    We visualize
    with our aligner
    the shared space 
    of
    DINOv2 ViT-B/14 and all-mpnet-base-v2 trained on 
    MS COCO-train
    images
    and 
    SPC
    captions. 
    Retrieval images and captions are from MS COCO-val. 
    Given a text, our aligner can retrieve matching images 
    (\subref{fig:qualitative_text_to_image})
    and vice versa (\subref{fig:qualitative_image_to_text}). 
    We 
    plot
    the first two principal components of the aligned space 
    (\subref{fig:qualitative_space}) 
    along
    with the MS COCO supercategories:
    classes coarsely align.
    Text-to-image generations
    (\subref{fig:text2img_small}),
    using a diffusion model 
    from
    embeddings to 
    images,
    show that
    our unpaired aligner 
    already generates images of the broad class and scene while additional pairs improve finer details.
    }
    \label{fig:qualitative}
\end{figure}

\inparagraph{Results.}
Results are in \cref{fig:qualitative}.
Both retrieval experiments show that coarse semantic properties, such as scene type and broad object category, are captured by the mapping. 
The point clouds colored by 
MS COCO superclasses further emphasize the coarse alignment. 
Though not every image is mapped exactly to its corresponding caption, 
categories are correctly mapped. 
We show similar plots for three other modalities with varying alignment quality in \cref{fig:domain_spectrum}. 
Our aligner can be successfully combined with a generative model to enable
a
full mapping from 
captions to 
images,
with reasonable quality in the unpaired setting, albeit with often missing fine details.
Adding a few pairs sharpens individual objects and scene attributes. 
We evaluate the generated images quantitatively and compare our aligner to a linear aligner, also with the Contriever model~\citep{izacard2022unsupervised} in \cref{app:text2img}.

\section{Discussion}

We 
have
shown 
that shared embedding geometrical
information
often 
suffices
for
cross-modal alignment.
Yet,
important 
conditions
for success
unavoidably remain.
We discuss three
of them here.

\inparagraph{What enables unpaired cross-modal alignment?}
Related semantics alone 
seem
insufficient:
both embedding spaces must also organize 
them
in a similar geometric structure,
as
supported by our 
experiments, where CKA strongly predict 
our aligner's performance.
Models
with more similar geometry are consistently 
better
aligned.
This
relationship holds 
experimentally
beyond vision and language, suggesting that the role of shared geometry is not specific to a particular pair of modalities.
Thus, geometric similarity 
explains
why unpaired alignment succeeds.
This 
also highlights an important limitation.
Currently, we cannot reliably determine from 
unpaired data alone whether two spaces are sufficiently similar to be aligned,
as
usual
alignment metrics
need
matching
samples.
Developing 
fully unsupervised
alignment prediction 
scores
remains an open problem.

\inparagraph{Why is the recovered alignment mostly coarse?}
The recovered alignment is often substantially better at matching coarse semantic structure than fine-grained details.
This is consistent with the observation that modalities share information about the underlying world while also encoding modality-specific information~\citep{koepke2026back}.
Our analysis of local versus coarse geometric alignment in \cref{app:alignment_granularity} supports this view: coarse clusters are more consistently aligned than local neighborhoods.
Thus, our aligner can recover shared semantic structure without necessarily recovering modality-specific details or exact sample-wise correspondences.

\inparagraph{How does 
data quality affect shared geometry?}
We
mostly
used image-caption datasets
with
relatively high-quality visual and textual descriptions.
Yet, much of 
web-available
text
does not describe visual scenes and may 
thus
share less semantic structure with vision representations.
This limits the applicability of our current approach to noisier multimodal data,
though
our results suggest that 
it
can be mitigated by improving
representation quality.
Language models 
already help identify text 
relevant to 
visual content~\citep{zhang2026multi}, and generative token pooling~\citep{wang2026the} improves geometric alignment for short captions (see \cref{app:gen_token_pooling}).
These results suggest that 
better filtering and representations
might
extend unpaired cross-modal alignment to noisier 
settings.

\section{Conclusion}
We showed that independently trained models 
across
modalities can often be aligned without any paired data.
A simple Wasserstein Procrustes method with a coarse geometric initialization (GeoWarP) recovers coarse cross-modal correspondences,
even when 
modalities are drawn from different datasets.
We
further showed 
alignment success is strongly predicted by CKA scores,
indicating that greater geometric similarity leads to better unpaired alignment.
These 
findings
are consistent across vision and language and extend to scientific and medical domains.
Finally, 
paired supervision
can be incorporated naturally into 
our
framework.
Our method substantially improves over existing approaches in the low-pair regime and remains competitive as more pairs become available.
Taken together, 
our
results suggest that a multimodal Rosetta Stone is not always necessary.
In many cases, the geometry learned independently by different models already contains enough information to connect their modalities.
Our work creates exciting opportunities for new modalities for which paired information is scarce or even non-existent, as long as sufficient geometry is shared between representations.

\subsubsection*{Acknowledgments}
We thank Nikita Araslanov and Riccardo Marin for their helpful suggestions.
This work was supported by a Packard Fellowship to P.I. Support from ONR MURI grants N00014-22-1-2740 and N00014-26-1-2024 and was partially funded by the German Federal Ministry of Education and Research through the ExperTeam4KI funding program for UDance (Grant No. 01IS24064).

\microtypesetup{expansion=true}
\bibliography{main}
\bibliographystyle{iclr2027_conference}
\microtypesetup{expansion=false}

\clearpage
\appendix

\section*{Appendix Overview}

\begin{itemize}[leftmargin=1.2em, itemsep=0.15em, topsep=0.2em]
    \item \Cref{app:unpaired_alignment_methods} describes our aligner, GeoWarP, in full and ablates every component and hyperparameter of the method.
    \item \Cref{app:experimental_details} lists the models, datasets, splits, prompts, metrics and their
          parameters for each experiment of the main paper, so that every number can be reproduced.
    \item \Cref{app:results} reports the results the main text summarizes, per dataset, per benchmark and per
          modality pair, together with the generated images of the text-to-image experiment.
    \item \Cref{app:gen_token_pooling} studies generative token pooling~\citep{wang2026the}, which uses the average of generated tokens 
          as the embedding for a given caption. 
          It helps on short captions and hurts on detailed ones, which explains why it fails on the detailed captioning corpora of  \cref{subsec:unpaired}.
    \item \Cref{app:alignment_granularity} measures which part of the geometry the two modalities actually share.
          The agreement sits in coarse rather than fine structures.
\end{itemize}

\section{GeoWarP: Details and Ablation}
\label{app:unpaired_alignment_methods}
This appendix provides details about our algorithm and ablates it against other components and hyperparameter choices.

\subsection{Our Algorithm}
\label{app:algorithm}
\Cref{alg:ours,alg:init} summarize our method, called \textbf{Geo}metric initialization and \textbf{Wa}sse\textbf{r}stein \textbf{P}rocrustes (GeoWarP).
It consists of three steps:
The initialization seeks a coarse transport that connects the two spaces, the read-out transforms the coarse transport into a sample-wise correspondence, and the refinement uses standard Wasserstein Procrustes alternation to improve the mapping.

\inparagraph{Notation.}
Let $\mX \in \R^{n \times d_X}$ and $\mY \in \R^{m \times d_Y}$, $d_X \leq d_Y$, be centered and normalized embeddings from independently trained models.
We use $\va = \vone_n / n$ and $\vb = \vone_m / m$ to denote the uniform marginals and $\Pi(\va, \vb) = \{\mT \in \R^{n \times m}_{\ge 0} : \mT\vone_m = \va,\ \mT^{\top}\vone_n = \vb\}$ for the transport polytope.
The vertices of the transport polytope for square matrices are permutation matrices $\mathcal{P}_n$.
Moreover, we use the Stiefel manifold $\St(d_X, d_Y)$, which is the set of semi-orthogonal matrices $\mW \in \R^{d_X \times d_Y}$ whose columns are orthogonal if $d_X \ge d_Y$ and whose rows are orthonormal otherwise.
We can transform a matrix into a semi-orthogonal matrix with the polar factor $\polar(\mM) = \mU\mV^{\top}$ using the singular value decomposition $\mM = \mU\mSigma\mV^{\top}$.
To ease the notation with Gromov-Wasserstein (GW) problems, we denote the linear kernels as $\mK_X = \mX\mX^{\top}$ and $\mK_Y = \mY\mY^{\top}$.
Finally, let $\mH = \mI - \tfrac{1}{n}\vone\vone^{\top}$ be the centering matrix that is used in the definition of the centered kernel alignment (CKA)~\citep{kornblith2019similarity}.

\inparagraph{Initialization.}
The Procrustes-Wasserstein problem is jointly non-convex.
Moreover, random initializations fail in practice~\citep{artetxe2018robust} and the uniform plan carries no information about the map, since on centered data $\mX^{\top}\va\vb^{\top}\mY = \bar{\vx}\bar{\vy}^{\top} = \vzero$, so the first Procrustes step has no signal.

We don't assume a perfect isometry but rather shared coarse structures in the embedding spaces.
Therefore, we use the coarse clustering initialization from mini-vec2vec~\citep{dar2025mini}.
This initialization repeatedly matches cluster centers to create a set of pseudo-pairs.
This makes the problem computationally tractable and relies on shared coarse structures.

For $s = 1, \dots, S$, we draw $b$ rows of $\mX$ and, independently, $b$ rows of $\mY$ and cluster each subset with $k$-means into $C$ clusters.
The resulting cluster centers $\mA_s \in \R^{C \times d_X}$ and $\mB_s \in \R^{C \times d_Y}$ are then matched by the permutation that maximizes their linear centered kernel alignment (CKA)~\citep{kornblith2019similarity}:
\begin{equation}
    \label{eq:cluster_qap}
    \mP_s \in \argmax_{\mP \in \mathcal{P}_C}\; \cka\left(\mA_s, \mP\mB_s\right).
    \end{equation}

The optimization of CKA over permutation matrices is a quadratic assignment problem (QAP), as the following lemma shows.
\begin{lemma}
\label{lem:cka_qap}
Let $\mA \in \R^{C \times d_X}$, $\mB \in \R^{C \times d_Y}$ contain $C$ samples. Then
\begin{equation}
    \label{eq:cka_qap}
    \argmax_{\mP \in \mathcal{P}_C} \cka(\mA, \mP\mB) = \argmax_{\mP \in \mathcal{P}_C} \Tr\left(\bar{\mK}_A\mP\bar{\mK}_B\mP^{\top}\right)
\end{equation}
is a Koopmans-Beckmann quadratic assignment problem~\citep{koopmans1957assignment} with $\bar{\mK}_A = \mH \mA \mA^{\top} \mH$ and $\bar{\mK}_B = \mH \mB \mB^{\top} \mH$ for $\mH = \mI - \tfrac{1}{C}\vone\vone^{\top}$.
\end{lemma}
\begin{proof}
Let $\mK_A = \mA \mA^{\top}$ and $\mK_B = \mB \mB^{\top}$ be the unnormalized kernels and $\mP \in \mathcal{P}_C$ a permutation matrix.
Given the definition of CKA, we have that 
\begin{equation*}
    \cka(\mA, \mP\mB) = \frac{\hsic(\mK_A, \mP \mK_B \mP^{\top})}{\sqrt{\hsic(\mK_A, \mK_A)\hsic(\mP \mK_B \mP^{\top}, \mP \mK_B \mP^{\top})}}.
\end{equation*}
First, we notice that the $\mH$ commutes with a permutation, \ie,
\begin{equation*}
    \mP \mH = \mP (\mI - \tfrac{1}{C}\vone\vone^{\top}) = \mP - \tfrac{1}{C}\vone\vone^{\top} = (\mI - \tfrac{1}{C}\vone\vone^{\top}) \mP = \mH \mP.
\end{equation*}
Plugged into the definition of HSIC, we have
\begin{align*}
    \hsic(\mP \mK_B \mP^{\top}, \mP \mK_B \mP^{\top})
    &= \tfrac{1}{(C - 1)^2} \Tr(\mP \mK_B \mP^{\top} \mH \mP \mK_B \mP^{\top} \mH) \\
    &= \tfrac{1}{(C - 1)^2} \Tr(\mK_B \mH \mP^{\top} \mP \mK_B \mP^{\top} \mP \mH),
\end{align*}
where we additionally use the cyclic property of the trace.
The permutations each cancel out because of $\mP^{\top} \mP = \mI$.
Thus, the denominator is independent of $\mP$.
Next, we can use that $\mH$ is idempotent,
\begin{equation*}
    \mH \mH = (\mI - \tfrac{1}{C}\vone\vone^{\top}) (\mI - \tfrac{1}{C}\vone\vone^{\top}) = \mI - \tfrac{1}{C}\vone\vone^{\top} - \tfrac{1}{C}\vone\vone^{\top} + \tfrac{1}{C^2}\vone\vone^{\top}\vone\vone^{\top} = \mI - \tfrac{1}{C}\vone\vone^{\top} = \mH.
\end{equation*}
Plugging this in the definitions leads to
\begin{equation*}
    \cka(\mA, \mP\mB) \propto \Tr(\mK_A \mH \mP \mK_B \mP^{\top} \mH) = \Tr(\mK_A \mH \mH \mP \mK_B \mP^{\top} \mH \mH) = \Tr(\bar{\mK}_A\mP\bar{\mK}_B\mP^{\top})
\end{equation*}
by the cyclic property of the trace.
\end{proof}
We use the dual-ascent solver MPOpt with their primal heuristic~\citep{hutschenreiter2021fusion}.
Each iteration implicitly optimizes a low-rank Gromov-Wasserstein (GW) problem.
GW seeks a transport plan that preserves within-space similarity:
\begin{equation}
\label{eq:gromov_wasserstein2}
\mT^{\text{GW}} \in \argmax_{\mT \in \Pi(\va, \vb)}\;
\Tr(\mK_X\mT\mK_Y\mT^\top).
\end{equation}
This problem is closely related to the Wasserstein Procrustes problem~\citep{alvarez2019towards} and the maximization of CKA:
\begin{lemma}
\label{lem:cka_gw}
Let $\mA \in \R^{C \times d_X}$, $\mB \in \R^{C \times d_Y}$ contain $C$ samples with unnormalized kernels $\mK_A = \mA \mA^{\top}$ and $\mK_B = \mB \mB^{\top}$.
Then
\begin{equation}
    \label{eq:cka_gw}
    \argmax_{\mP \in \mathcal{P}_C} \cka(\mA, \mP\mB) = \argmax_{\mP \in \mathcal{P}_C} \Tr\left(\mK_A\mP\mK_B\mP^{\top}\right) - \frac{2}{C} \Tr(\vk \vl^{\top} \mP^{\top}),
\end{equation}
where $\vk = \mK_A \vone$ and $\vl = \mK_B \vone$ are the row sums of the unnormalized kernels.
Maximizing CKA over permutations is therefore the Gromov-Wasserstein problem of \cref{eq:gromov_wasserstein2},
restricted to permutations, with the additional linear term $-\tfrac{2}{N}\Tr(\vk\vl^{\top}\mP^{\top})$.
\end{lemma}
\begin{proof}
For a permutation $\mP \in \mathcal{P}_C$, we have that $\cka(\mA, \mP\mB) \propto \Tr(\mK_A \mH \mP \mK_B \mP^{\top} \mH)$, analogously to \cref{lem:cka_qap}.
Plugging in the definition of $\mH = \mI - \tfrac{1}{C}\vone\vone^{\top}$ and expanding, we get
$\Tr(\mK_A \mH \mP \mK_B \mP^{\top} \mH) = \Tr(\mK_A (\mI - \tfrac{1}{C}\vone\vone^{\top}) \mP \mK_B \mP^{\top} (\mI - \tfrac{1}{C}\vone\vone^{\top})) = \Tr(\mK_A \mP \mK_B \mP^{\top} - \frac{1}{C} \mK_A \vone \vone^{\top} \mP \mK_B \mP^{\top} - \frac{1}{C} \mK_A \mP \mK_B \mP^{\top} \vone \vone^{\top} + \frac{1}{C^2} \mK_A \vone \vone^{\top} \mP \mK_B \mP^{\top} \vone \vone^{\top}) = \Tr(\mK_A \mP \mK_B \mP^{\top}) - \frac{1}{C} \Tr(\vk \vl^{\top} \mP^{\top}) - \frac{1}{C} \Tr(\vk^{\top} \mP \vl ) + \frac{1}{C^2} \vone^{\top} \vk \vone^{\top} \vl = \Tr\left(\mK_A\mP\mK_B\mP^{\top}\right) - \frac{2}{C} \Tr(\vk \vl^{\top} \mP^{\top}) + \text{const}$,
where const is a constant that does not depend on $\mP$.
We use that $\mP \vone = \vone$, $\mP^{\top} \vone = \vone$ and the cyclic property of the trace.
\end{proof}
In particular, for our centered and normalized embeddings, both $\vl$ and $\vk$ are small, and the maximization over the linear CKA approximately corresponds to a maximization of the GW problem.
Given the cluster assignments $\mZ_X \in \{0, 1\}^{n \times C}$ and $\mZ_Y \in \{0, 1\}^{m \times C}$, and the cluster sizes $\vn = \mZ_X^{\top}\vone$ and $\vm = \mZ_Y^{\top}\vone$, we can recover a sample-wise transport matrix with
\begin{equation}
    \label{eq:lowrank_plan}
    \mT_{\mP_s} = \tfrac{1}{C}\,\mZ_X\,\mathrm{diag}(\vn)^{-1}\,\mP_s\,\mathrm{diag}(\vm)^{-1}\,\mZ_Y^{\top}.
\end{equation}
Substituting it into \cref{eq:gromov_wasserstein2} reduces the $n \times m$ problem to the coarse QAP over the $C$ centers with the cross-covariance to $ \mX^{\top}\mT_{\mP_s}\mY = \tfrac{1}{C}\,\mA_s^{\top}\mP_s\mB_s$.
Moreover, it is a transport plan with rank $C$.
This also shows a resemblance to a traditional initialization in low-rank GW.
\citet{scetbon2022linear} introduce an initialization, where both spaces are initialized with k-means.
However, both clusters are paired using the arbitrary order from k-means, and no explicit permutation is optimized:
\begin{equation}
    \label{eq:lowrank_original}
    \mT = \mZ_X\,\mathrm{diag}(1/\vg)\,\mZ_Y^{\top}.
\end{equation}
Instead, they optimize all three factors with mirror descent.
We compare all low-rank GW initializations in \cref{app:ablation_components} and observe that for structured data, the clustering and matching approach from mini-vec2vec leads to a much better initialization and final optimum.

As a last step, all transport matrices are averaged and captured in the cross-covariance $\mM=\frac{1}{SC}\sum_s\mA_s^\top\mP_s\mB_s = \mX^{\top} \left(\frac{1}{S}\sum_s \mT_{\mP_s}\right)\mY$.

\inparagraph{Read-out.}
Each initialization is a correspondence between clusters, and the read-out turns them into a correspondence between samples.
Given the averaged transport matrix, we take one conditional-gradient step on the Gromov-Wasserstein objective.
The gradient of \cref{eq:gromov_wasserstein2} at $\bar{\mT} = \frac{1}{SC}\sum_s \mT_{\mP_s}$ is $2\mK_X\bar{\mT}\mK_Y = 2\mX\mM\mY^{\top}$.
A conditional-gradient step maximizes the linearization over the transport polytope, \ie $\mT_0 = \hungarian\left(\mX\mM\mY^{\top}\right)$.
From this, we can get an initial map with orthogonal Procrustes: $\mW = \polar\left(\mX^{\top}\mT_0\mY\right)$.
An assignment between all $n$ and $m$ samples is cubic in their number, so we shuffle both sets once, cut them into consecutive blocks of $b$ rows, and assign the $i$-th block of $\mX$ to the $i$-th block of $\mY$.
This yields $\min(n, m)$ pseudo-pairs from $\lceil \min(n, m) / b \rceil$ small assignments each of size at most $b \times b$.

\inparagraph{Refinement.}
The refinement follows the standard alternation for Wasserstein Procrustes on batches.
For $r = 1, \dots, R$, we draw independent batches $\mX_B$ and $\mY_B$ of $b$ rows and alternate the assignment $\mT = \hungarian\left(\mX_B\mW\mY_B^{\top}\right)$ with the Procrustes update $\mW = \polar\left(\mX_B^{\top}\mT\mY_B\right)$.

\inparagraph{Retrieval.}
During validation, we simply use the inner product between the mapped source and target embedding, $\vx\mW\vy^{\top}$.
Because the embeddings are normalized, this is the same as their cosine similarity.

\inparagraph{Known pairs.}
Pairs can be incorporated naturally into our algorithm.
Given $p$ pairs, $(\hat{\vx}_i)_{i=1}^p = \hat{\mX} \in \R^{p \times d_X}$ and $(\hat{\vy}_j)_{j=1}^p = \hat{\mY} \in \R^{p \times d_Y}$, we can treat them similar to the unpaired samples but with a known correspondence.
During the QAP, known pairs enter as a linear term:
\begin{lemma}
\label{lem:cka_pairs}
Let $\mA \in \R^{C \times d_X}$, $\mB \in \R^{C \times d_Y}$ be $C$ unpaired samples and let $\hat{\mX} \in \R^{p \times d_X}$, $\hat{\mY} \in \R^{p \times d_Y}$ be $p$ paired samples.
We collect all samples in $\tilde{\mA} = \begin{pmatrix}\mA\\\hat{\mX}\end{pmatrix}$ and $\tilde{\mB} = \begin{pmatrix}\mB\\\hat{\mY}\end{pmatrix}$ and their centered kernel as $\tilde{\mK} = \mH \tilde{\mA} \tilde{\mA}^{\top} \mH$ and $\tilde{\mL} = \mH \tilde{\mB} \tilde{\mB}^{\top} \mH$, partitioned into $\tilde{\mK}_{11} \in \R^{C \times C}$, $\tilde{\mK}_{12} \in \R^{C \times p}$, and $\tilde{\mK}_{22} \in \R^{p \times p}$ and likewise for $\tilde{\mL}$.
Then
\begin{equation}
    \label{eq:cka_pairs}
    \argmax_{\mP \in \mathcal{P}_C} \cka\left(\begin{pmatrix}\mA\\\hat{\mX}\end{pmatrix}, \begin{pmatrix}\mP \mB\\\hat{\mY}\end{pmatrix}\right) = \argmax_{\mP \in \mathcal{P}_C} \Tr\left(\tilde{\mK}_{11}\mP\tilde{\mL}_{11}\mP^{\top}\right) + 2 \Tr\left(\tilde{\mK}_{12} \tilde{\mL}_{12}^{\top} \mP^{\top}\right).
\end{equation}
Therefore, the pairs only enter during centering and as a linear term.
\end{lemma}
\begin{proof}
Let $\tilde{\mP} = \diag(\mP, \mI) \in \mathcal{P}_{C + p}$ be the joint permutation matrix for a permutation $\mP \in \mathcal{P}_C$.
Then, similar to \cref{lem:cka_qap},
\begin{align*}
\cka(\begin{pmatrix}\mA\\\hat{\mX}\end{pmatrix}, \begin{pmatrix}\mP \mB\\\hat{\mY}\end{pmatrix}) &= \cka(\tilde{\mA}, \tilde{\mP}\tilde{\mB})\\
&\propto \Tr(\tilde{\mK}\tilde{\mP}\tilde{\mL}\tilde{\mP}^{\top})\\
&=\Tr\left(\begin{pmatrix}\tilde{\mK}_{11} & \tilde{\mK}_{12} \\ \tilde{\mK}_{12}^{\top} & \tilde{\mK}_{22}\end{pmatrix} \begin{pmatrix} \mP & \bm{0} \\ \bm{0} & \mI \end{pmatrix} \begin{pmatrix}\tilde{\mL}_{11} & \tilde{\mL}_{12} \\ \tilde{\mL}_{12}^{\top} & \tilde{\mL}_{22} \end{pmatrix} \begin{pmatrix} \mP^{\top} & \bm{0} \\ \bm{0} & \mI \end{pmatrix}\right)\\
&=\Tr\left(\begin{pmatrix}\tilde{\mK}_{11} & \tilde{\mK}_{12} \\ \tilde{\mK}_{12}^{\top} & \tilde{\mK}_{22}\end{pmatrix} \begin{pmatrix} \mP \tilde{\mL}_{11} \mP^{\top} & \mP \tilde{\mL}_{12} \\ \tilde{\mL}_{12}^{\top} \mP^{\top} & \tilde{\mL}_{22} \end{pmatrix}\right)\\
&=\Tr(\tilde{\mK}_{11} \mP \tilde{\mL}_{11} \mP^{\top} + \tilde{\mK}_{12} \tilde{\mL}_{12}^{\top} \mP^{\top}) + \Tr(\tilde{\mK}_{12}^{\top} \mP \tilde{\mL}_{12} + \tilde{\mK}_{22} \tilde{\mL}_{22})\\
&= \Tr(\tilde{\mK}_{11}\mP\tilde{\mL}_{11}\mP^{\top}) + 2 \Tr(\tilde{\mK}_{12} \tilde{\mL}_{12}^{\top} \mP^{\top}) + \text{const},
\end{align*}
where const is a constant that does not depend on $\mP$.
\end{proof}
The MPOpt solver natively handles a linear term.
This also means that the complexity of the solver doesn't depend on the number of pseudo pairs.
During the Procrustes steps, the pairs also enter as a linear term and are weighted so that they carry the same total weight as the unpaired samples in the batch.

\inparagraph{Reproducibility.}
Embeddings are stored in bfloat16, normalized to unit length, and cast to float32.
The aligner subtracts the mean of its own training set and normalizes every row again.
All random subsets are drawn without replacement, and every run is seeded once for Python, NumPy, and PyTorch with deterministic algorithms enabled.
The clustering is scikit-learn's \texttt{KMeans} with its default parameters, which include $\kmeans$++ seeding, and the labels are reassigned to the nearest center afterwards.
MPOpt is called through \texttt{pylibmgm} 1.1.2 in the variant whose primal heuristic is improved by the local search, with all costs shifted to be non-positive, which changes the objective by a constant only, and with a batch size of 100, 10 greedy generations, and the stopping rule $p = 0.1$, $k = 100$ after at most 1000 batches.
The assignments of the read-out and the refinement use SciPy's \texttt{linear\_sum\_assignment}, a Jonker-Volgenant algorithm~\citep{jonker1988shortest}, on float64 cost matrices.
The polar factor is computed from the SVD of PyTorch and falls back to float64 when the float32 decomposition fails to converge.
We use PyTorch 2.14, scikit-learn 1.9, SciPy 1.18, and five seeds for every reported number unless stated otherwise.

\inparagraph{Computational cost.}
Let $I$ be the number of Lloyd iterations and $d = \max(d_X, d_Y)$.
One clustering run costs $\mathcal{O}(bCdI)$ for the two $\kmeans$ and $\mathcal{O}(C^{2}d)$ for the kernels, and the
QAP solver runs in $\mathcal{O}(C^{4})$, which takes seconds for $C = 30$.
The $S$ runs are independent.
The read-out and the refinement are $\lceil \min(n, m) / b\rceil + R$ assignments of size $b$, each preceded by a
score matrix in $\mathcal{O}(b^{2}d)$ and solved in $\mathcal{O}(b^{3})$ in the worst case, and each followed by a
Procrustes step in $\mathcal{O}(d_X d_Y \min(d_X, d_Y))$.
No stage forms an $n \times m$ matrix.
The memory footprint is $\mathcal{O}(b^{2} + d_X d_Y)$ beyond the embeddings themselves, and the running time is
linear in $n$ through the number of read-out blocks and otherwise independent of the data size.

\subsection{Ablation}
\label{app:ablation}
We perform two ablation studies: one study to compare our choice of components with other choices made in the literature (\cref{app:ablation_components}) and one study showing the robustness of our four hyperparameters (\cref{app:ablation_hyperparameters}). We evaluate on DINOv2 ViT-B/14~\citep{Oquab:2024:DIN} with Qwen3-8B~\citep{yang2025qwen3} and generative token pooling~\citep{wang2026the} on MS COCO~\citep{chen2015microsoft}, GTR~\citep{ni2022large} with GTE~\citep{li2023towards} on NQ~\citep{kwiatkowski2019natural}, and the RNA and ATAC embeddings of SNARE-seq~\citep{chen2019high}, and report FOSCTTM over five seeds.

\subsubsection{Component Ablation}
\label{app:ablation_components}
Our approach consists of three stages: initialization, read-out, and refinement. Most other methods can be divided similarly. In general, both adversarial and optimization-based methods search for a map $f$ from one embedding space to the other under which the source cloud is close to the target cloud~\citep{hartmann2019comparing}:
\begin{equation}
    \label{eq:general_objective}
    f^{\star} \in \argmin_{f \in \mathcal{F}}\; D\left(f(\mX), \mY\right).
\end{equation}
Adversarial methods usually approximate the distance with discriminators, corresponding to the Jensen-Shannon divergence for the standard GAN loss~\citep{goodfellow2014gan}, Wasserstein-1 distance for Wasserstein-GAN~\citep{arjovsky2017wasserstein}, or any $f$-divergence~\citep{nowozin2016f}. This gives adversarial methods the advantage that the function space $\mathcal{F}$ can be chosen more freely. Nonetheless, most of them choose simple MLPs~\citep{jha2026harnessing}, linear transformations~\citep{hoshen2018unsupervised, zhang2017earth, xu2018unsupervised}, or orthogonal transformations~\citep{conneau2017word, zhang2017adversarial}. Finally, the adversarial loss is often only the initialization and is usually further refined with optimization-based refinements~\citep{conneau2017word, zhang2017earth, mohiuddin2019revisiting}. In general, recent work~\citep{jha2026harnessing, dar2025mini} finds that adversarial methods are generally very unstable, and we can confirm this observation even more strongly in our unpaired experiments (\cref{subsec:unpaired}). Therefore, we mainly consider optimization-based methods in our ablation study.

\begin{table}[ht]
    \centering\scriptsize
    \begin {tabular}{lS[table-format=1.3]@{\,\ensuremath {\pm }\,}S[table-format=1.3]S[table-format=1.3]@{\,\ensuremath {\pm }\,}S[table-format=1.3]S[table-format=1.3]@{\,\ensuremath {\pm }\,}S[table-format=1.3]}%
\toprule Initialization&\multicolumn {2}{c}{NQ}&\multicolumn {2}{c}{SNARE-seq}&\multicolumn {2}{c}{MS COCO}\\\midrule %
PCA heuristic&0.503&0.055&0.321&0.234&0.265&0.116\\%
sorted heuristic&0.499&0.034&0.351&0.134&0.472&0.010\\%
Gromov-Wasserstein (\allowbreak {}random)&0.455&0.039&0.209&0.129&0.432&0.153\\%
Gromov-Wasserstein (\allowbreak {}uniform)&0.527&0.042&0.222&0.128&0.316&0.054\\%
mini-vec2vec (\allowbreak {}\ensuremath {k}-means + 2-opt, \ensuremath {C=\allowbreak {}20})&0.082&0.071&\protect \bfseries  0.194&\protect \bfseries  0.041&0.092&0.016\\%
mini-vec2vec (\allowbreak {}\ensuremath {k}-means + 2-opt, \ensuremath {C=\allowbreak {}30})&0.070&0.049&0.221&0.007&0.092&0.019\\%
ours (\allowbreak {}\ensuremath {k}-means + MPOpt)&\protect \bfseries  0.000&\protect \bfseries  0.000&0.216&0.034&\protect \bfseries  0.076&\protect \bfseries  0.023\\\bottomrule %
\end {tabular}

    \caption{\textbf{Initialization.} FOSCTTM of every initialization with the same read-out and no refinement. The clustering and matching initialization outperforms all other initializations. For NQ and MS COCO, MPOpt outperforms the 2-opt solver from mini-vec2vec while being competitive on SNARE-seq.
    }
    \label{tab:initialization_stage}
\end{table}

\inparagraph{Initialization.}
The Wasserstein Procrustes problem is jointly non-convex~\citep{alvarez2019towards} and Gromov-Wasserstein is NP-hard in general~\citep{scetbon2022linear}. Therefore, a variety of heuristics have been introduced to initialize the correspondence. \citet{artetxe2018robust} sorts each row of both kernel square roots to find a correspondence between the distance distributions (sorted heuristic), and \citet{hoshen2018non} takes inspiration from 3D point cloud matching and matches the leading PCA dimensions (PCA heuristic). We also compare a random and uniform transport initialization refined with conditional gradient~\citep{Peyre:2016:GWA} on 1024 samples -- Gromov-Wasserstein (random) and Gromov-Wasserstein (uniform) -- and the original mini-vec2vec initialization with 2-opt and two different numbers of clusters. We do not perform any refinement and use the same readout on all of them.

On the small SNARE-seq benchmark, all solvers are competitive, but the two heuristics are slightly worse. For the NQ and MS COCO datasets, a clear pattern is observable. On NQ, all heuristics and Gromov-Wasserstein initializations perform near chance, and on MS COCO they stay far behind matching cluster centers. Matching cluster centers makes the problem solvable, and the quality of the solver mainly determines the quality of the alignment.

\begin{figure}[t]
    \centering
    \begin{subfigure}[b]{0.56\textwidth}
        \centering
        \includegraphics{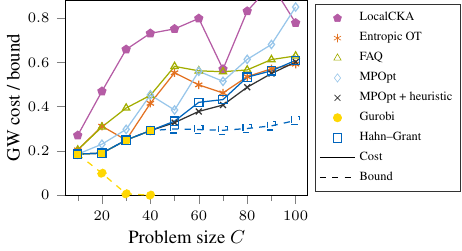} %
        \caption{Cost and bound against the problem size.}
        \label{fig:qap_solvers}
    \end{subfigure}%
    \hfill
    \begin{subfigure}[b]{0.42\textwidth}
        \centering\scriptsize
        \setlength{\tabcolsep}{2.5pt}
        \begin{tabular}{@{}l S[table-format=1.3] S[table-format=4.0] S[table-format=1.3] S[table-format=5.0] S[table-format=1.3]@{}}
            \toprule
            & \multicolumn{2}{c}{Hahn-Grant} & \multicolumn{3}{c}{MPOpt + heuristic} \\
            \cmidrule(lr){2-3}\cmidrule(l){4-6}
            Size & {cost} & {time [s]} & {cost} & {time [s]} & {equal time} \\
            \midrule
            10  & \bfseries 0.186 & \bfseries 1 & \bfseries 0.186 & \bfseries 1 & \bfseries 0.186 \\
            20  & \bfseries 0.190 & 15 & \bfseries 0.190 & \bfseries 6 & \bfseries 0.190 \\
            30  & \bfseries 0.249 & 84 & \bfseries 0.249 & \bfseries 45 & \bfseries 0.249 \\
            40  & \bfseries 0.292 & 389 & \bfseries 0.292 & \bfseries 108 & \bfseries 0.292 \\
            50  & 0.335 & 5400 & \bfseries 0.328 & \bfseries 417 & \bfseries 0.328 \\
            60  & 0.420 & 5400 & \bfseries 0.379 & \bfseries 1349 & \bfseries 0.379 \\
            70  & 0.432 & 5400 & 0.408 & \bfseries 1915 & \bfseries 0.406 \\
            80  & 0.533 & 5400 & \bfseries 0.488 & \bfseries 4563 & \bfseries 0.488 \\
            90  & 0.562 & \bfseries 5400 & \bfseries 0.550 & 6755 & \bfseries 0.550 \\
            100 & 0.609 & \bfseries 5400 & \bfseries 0.602 & 13939 & 0.604 \\
            \bottomrule
        \end{tabular}
        \vspace{0.2em}
        \caption{Cost and time of the two best solvers.}
        \label{tab:qap_times}
    \end{subfigure}
    \caption{\textbf{QAP solvers on the class-matching benchmark of \citet{schnaus2025it}.}
    (\subref{fig:qap_solvers})~Gromov-Wasserstein cost of the returned permutation (solid) and lower bound of the
    solver (dashed), normalized by the squared problem size. MPOpt with the primal
    heuristic, the solver of our initialization, matches them there and is best on most of the larger sizes. Its bound lies below the
    axis. (\subref{tab:qap_times})~Normalized cost, lower is better, and wall-clock time on one CPU, the best
    cost and the best time of every size in bold. The Hahn-Grant solver is stopped after 5400 seconds. The last
    column runs MPOpt with its primal heuristic for as long as the Hahn-Grant solver. All solvers except MPOpt with its primal heuristic are the
    published runs of \citet{schnaus2025it}.}
    \label{fig:qap_solver_benchmark}
\end{figure}

\inparagraph{QAP solver.}
The initialization ablation already motivates that the quality of the QAP solver plays a major role in the alignment quality. \citet{schnaus2025it} introduce their own QAP solver and benchmark against a variety of other solvers, including MPOpt without the primal heuristic. We repeat their experiment (\citet{schnaus2025it}, Fig.~6) that matches aggregated vision and language embeddings with the MPOpt solver with the local search. 

We show the results in black in \cref{fig:qap_solvers} and show the runtime, the GW cost, and the cost with the same runtime as the Hahn-Grant solver in \cref{tab:qap_times}. The other solvers are the published runs of \citet{schnaus2025it}, made with their official code. The local search improves the MPOpt solver by a large margin. With the settings of our initialization, MPOpt returns the certified optimum up to 40 classes in less time than the Hahn-Grant algorithm. Above 40 classes, it finds a better permutation than the Hahn-Grant algorithm, both with these settings and with a number of iterations that gives it about as much time as the Hahn-Grant algorithm.

\begin{table}[ht]
    \centering\scriptsize
    \begin {tabular}{lllllll}%
\toprule Initialization&blobs $k$=10&blobs $k$=30&curve&SNARE-seq&splatter&uniform\\\midrule %
random&0.0781 / 0.0590&0.0493 / 0.0412&0.0878 / 0.0319&0.0982 / 0.0795&0.0917 / 0.0765&0.0333 / 0.0278\\%
trivial&0.0652 / 0.0336&0.0412 / 0.0238&0.0732 / 0.0323&0.0818 / 0.0814&0.0765 / 0.0765&0.0278 / 0.0264\\%
\ensuremath {k}-means&0.0525 / 0.0373&0.0343 / 0.0241&0.0617 / 0.0343&0.0757 / 0.0593&0.0674 / 0.0568&0.0264 / 0.0227\\%
lower bound&0.0430 / 0.0378&0.0245 / 0.0241&0.0277 / 0.0279&0.0773 / 0.0612&0.0753 / 0.0636&\textbf {0.0198 / 0.0198}\\%
\ensuremath {k}-means + QAP&\textbf {0.0211 / 0.0206}&\textbf {0.0183 / 0.0168}&\textbf {0.0129 / 0.0129}&\textbf {0.0647 / 0.0589}&\textbf {0.0446 / 0.0446}&0.0251 / 0.0226\\\bottomrule %
\end {tabular}

    \caption{\textbf{Cluster matching as a low-rank Gromov-Wasserstein solver.} Gromov-Wasserstein cost of the
    low-rank plan before and after mirror descent (\emph{before} / \emph{after}), averaged over seeds and ranks. Lower is
    better. Matching cluster centers leads to a better initialization and a better optimum after mirror descent for structured data.}
    \label{tab:lowrank_objective}
\end{table}

\inparagraph{Low-rank Gromov-Wasserstein.}
In general, any low-rank GW solution could be used in our method as a coarse correspondence. Here, we evaluate the clustering and matching approach from mini-vec2vec on the toy examples from \citet{scetbon2022linear}. We report the GW objective before and after mirror descent, averaged over seeds and the ranks $10$, $20$, $30$, and $50$.

In \cref{tab:lowrank_objective}, we observe that matching cluster centers is consistently superior both in terms of the initial and final objective, except for the uniform toy example. In general, with more structure, this difference becomes bigger.

\begin{table}[ht]
    \centering\scriptsize
    \begin {tabular}{lS[table-format=1.3]@{\,\ensuremath {\pm }\,}S[table-format=1.3]S[table-format=1.3]@{\,\ensuremath {\pm }\,}S[table-format=1.3]S[table-format=1.3]@{\,\ensuremath {\pm }\,}S[table-format=1.3]}%
\toprule Read-out&\multicolumn {2}{c}{NQ}&\multicolumn {2}{c}{SNARE-seq}&\multicolumn {2}{c}{MS COCO}\\\midrule %
direct&0.024&0.007&0.246&0.050&0.091&0.020\\%
top-\ensuremath {k} relative representations&\protect \bfseries  0.000&\protect \bfseries  0.000&0.254&0.101&0.085&0.023\\%
batched Hungarian matching&\protect \bfseries  0.000&\protect \bfseries  0.000&\protect \bfseries  0.216&\protect \bfseries  0.034&\protect \bfseries  0.076&\protect \bfseries  0.023\\\bottomrule %
\end {tabular}

    \caption{\textbf{Read-out.} FOSCTTM after reading a map from the same averaged correspondence, without
    refinement. Our batched Hungarian matching slightly outperforms the other two methods while requiring fewer hyperparameters than top-k relative representations.}
    \label{tab:readout_stage}
\end{table}

\inparagraph{Read-out.}
Given a coarse transport plan, there are multiple methods to extend it to the sample level. \citet{Alvarez-Melis:2018:GWA} apply one Procrustes step to the plan, which for our plan is $\polar(\mM)$ and which we call the direct read-out. mini-vec2vec proposes the direct read-out and a read-out through relative representations~\citep{Moschella:2023:RRE}. It compares the relative representations of a source and a target by their cosine and averages the $k$ nearest targets of every source. Our read-out is one conditional-gradient step on the GW objective from $\bar{\mT}$, as described in \cref{app:algorithm}. The relative representation and the conditional gradient step implicitly optimize similar objectives. This can be seen when expanding $\left(\mX\mM\mY^{\top}\right)_{ij} = \tfrac{1}{SC}\left\langle \left(\mX\mA_{1:S}^{\top}\right)_i, \left(\mY\mB_{1:S}^{\top}\right)_j\right\rangle$. Hence, the mini-vec2vec read-out differs mainly in that it uses cosine similarities instead of inner products and the top-$k$ neighbors instead of a batched linear assignment.

In \cref{tab:readout_stage}, we compare all three read-out methods, each without refinement and with the same pseudo-pairs from our method. Our batched Hungarian matching is the best of the three on all benchmarks, and the direct read-out falls behind on NQ. Compared to the top-$k$ relative representation refinement from mini-vec2vec, it reduces one hyperparameter.

\begin{table}[ht]
    \centering\scriptsize
    \begin {tabular}{lS[table-format=1.3]@{\,\ensuremath {\pm }\,}S[table-format=1.3]S[table-format=1.3]@{\,\ensuremath {\pm }\,}S[table-format=1.3]S[table-format=1.3]@{\,\ensuremath {\pm }\,}S[table-format=1.3]}%
\toprule Refinement&\multicolumn {2}{c}{NQ}&\multicolumn {2}{c}{SNARE-seq}&\multicolumn {2}{c}{MS COCO}\\\midrule %
no refinement&\protect \bfseries  0.000&\protect \bfseries  0.000&0.216&0.034&0.076&0.023\\%
mini-vec2vec&\protect \bfseries  0.000&\protect \bfseries  0.000&\protect \bfseries  0.209&\protect \bfseries  0.033&0.046&0.039\\%
ours&\protect \bfseries  0.000&\protect \bfseries  0.000&\protect \bfseries  0.209&\protect \bfseries  0.034&\protect \bfseries  0.020&\protect \bfseries  0.015\\\bottomrule %
\end {tabular}

    \caption{\textbf{Refinement.} FOSCTTM after refining the same read-out. ``mini-vec2vec'' is its
    \emph{refine 1} and \emph{refine 2}~\citep{dar2025mini}, and ``ours'' is the batch alternation of \cref{alg:ours}. Our refinement leads to similar or better numbers on all three benchmarks. It also doesn't introduce any additional hyperparameters.}
    \label{tab:refinement_stage}
\end{table}

\inparagraph{Refinement.}
For the refinement, we compare our standard Wasserstein Procrustes alternation with the two-stage refinement from mini-vec2vec. The first refinement replaces the transport with the average over the top-$k$ neighbors, and it updates the weight with an exponential moving average of the Procrustes solutions. Compared to our refinement, this adds two additional hyperparameters and produces an average over semi-orthogonal matrices, which is not semi-orthogonal in general. The second refinement clusters the source space with $k$-means and uses the mapped cluster centers as initialization for clustering in the target space. This again corresponds to a low-rank transport plan, and an exponential moving with Procrustes is used for the weight update. This adds two further hyperparameters, namely the number of refinement steps with the second refinement and the number of clusters in the second refinement.

\Cref{tab:refinement_stage} compares the different refinements and shows that both refinements beat no refinement on MS COCO, while NQ and SNARE-seq are already saturated by the read-out. Moreover, on MS COCO, our refinement beats mini-vec2vec despite its simplicity.

\begin{figure}[t]
    \centering
    \begin{subfigure}[t]{0.49\textwidth}
        \centering
        \includegraphics{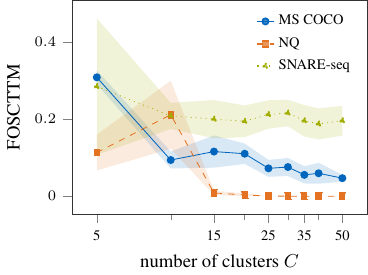} {} 
        \caption{Clusters $C$}
        \label{fig:cluster_sweep}
    \end{subfigure}
    \hfill
    \begin{subfigure}[t]{0.49\textwidth}
        \centering
        \includegraphics{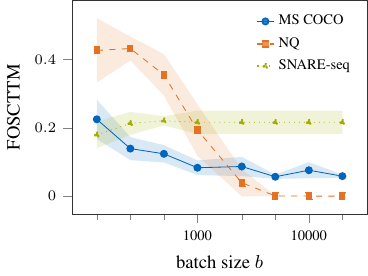} {} 
        \caption{Batch size $b$}
        \label{fig:batch_sweep}
    \end{subfigure}

    \vspace{0.8em}
    \begin{subfigure}[t]{0.49\textwidth}
        \centering
        \includegraphics{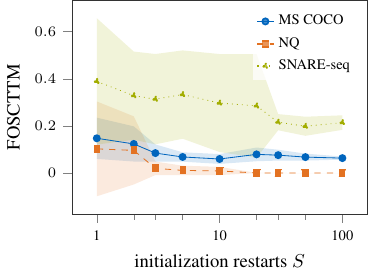} {} 
        \caption{Restarts $S$}
        \label{fig:restart_sweep}
    \end{subfigure}
    \hfill
    \begin{subfigure}[t]{0.49\textwidth}
        \centering
        \includegraphics{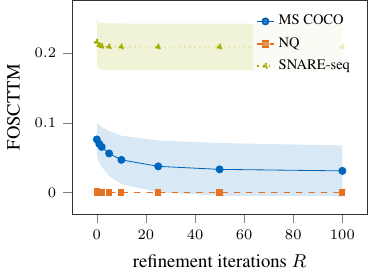} {} 
        \caption{Refinement iterations $R$}
        \label{fig:iteration_sweep}
    \end{subfigure}

    \caption{\textbf{All four hyperparameters improve with scale.} We ablate one variable at a time while fixing all other variables. The sweeps over the batch size and the restarts use the read-out without refinement ($R = 0$), so they isolate the initialization. Increasing each hyperparameters improves performance, except for the batch size on SNARE-seq, which saturates because it only contains around $520$ samples. In addition to an improved performance, more restarts of the initialization also reduce the variance.}
    \label{fig:hyperparameters}
\end{figure}

\subsubsection{Hyperparameter Ablation}
\label{app:ablation_hyperparameters}
The last section shows that our method outperforms previous methods while being simpler. It only contains four hyperparameters: the number of clusters $C$, the batch size $b$, and the number of iterations during initialization $S$ and refinement $R$. This section shows that our method is stable with respect to all of them, and choosing them larger generally improves the performance while trading off speed.

\inparagraph{Number of clusters ($C=30$).}
The number of clusters is the main hyperparameter of the initialization.
A larger number of clusters allows the initialization to approximate the Gromov-Wasserstein objective better, but it also makes the QAP harder to solve.

\Cref{fig:qap_solver_benchmark} shows that problems up to $C = 40$ clusters are solved close to optimality, and \citet{schnaus2025it} shows that the global optimum leads to better solutions than a local optimum.
We evaluate the effect of the number of clusters on the initialization in \cref{fig:cluster_sweep}.
In this experiment, we use our batched Hungarian read-out and no refinement.
We observe that an increasing number of clusters generally improves initialization, but the effect saturates at different points across benchmarks.
Altogether, the results suggest that $C = 30$ is a reasonable default across domains, balancing performance and computational cost even though larger values might lead to even better performance.

\inparagraph{Batch size ($b=10{,}000$).}
The batch size is used uniformly for all three stages and decides how many samples are used to approximate the whole problem.

\Cref{fig:batch_sweep} shows the performance as a function of the batch size. For SNARE-seq, a small batch size of $125$ is already enough to reach a good performance, and it doesn't degrade much for larger batch sizes. On NQ and MS COCO, larger batch sizes lead to better numbers. Both datasets are magnitudes larger than SNARE-seq, so a small batch only represents a small part of the data. For both datasets, the performance saturates at around a batch size of $5{,}000$. The batch size of $b=10{,}000$ is a conservative choice that becomes less sensitive as it increases.

\inparagraph{Number of initialization iterations ($S=30$).}
Each initialization iteration clusters its own subsample and matches the cluster centers, and the initialization
averages the resulting correspondences.
\Cref{fig:restart_sweep} shows the performance as a function of $S$.
We use the same read-out and no refinement to isolate the effect of the initialization.
More iterations generally improve the performance on all three datasets and strongly reduce the spread over the seeds.
On MS COCO, the FOSCTTM decreases from $0.148 \pm 0.088$ with a single iteration to $0.076 \pm 0.023$ with
$S=30$ and $0.064 \pm 0.006$ with $S=100$.
On NQ, it decreases from $0.103 \pm 0.201$ to $0.0003 \pm 0.0002$ with $S=30$, and on SNARE-seq from
$0.389 \pm 0.268$ to $0.216 \pm 0.034$.

\inparagraph{Number of refinement iterations ($R=100$).}
The refinement alternates between assignment and weight update on batches.

\Cref{fig:iteration_sweep} varies the number of iterations $R$. On NQ, refinement doesn't make a difference because the read-out already yields near-perfect alignment. On MS COCO, more iterations improve the performance monotonically, while SNARE-seq improves only slightly, within the spread over the seeds. Altogether, a conservative choice of $R=100$ works well for all datasets.

\FloatBarrier
\section{Experimental Details}
\label{app:experimental_details}

This appendix gives the models, data, splits and metric parameters of every experiment in \cref{sec:experiments}.
Together with the description of the aligner in \cref{app:algorithm} it contains everything needed to reproduce
the reported numbers.

\inparagraph{General setup.}
Every experiment fits the aligner of \cref{app:algorithm} with the same hyperparameters, $C = 30$ clusters,
$S = 30$ restarts, a batch size of $b = 10^4$ and $R = 100$ refinement iterations.
Each side is preprocessed by subtracting the mean of its own training samples and normalizing every row to unit
length, with a numerical floor of $10^{-10}$.
The two training sets are disjoint.
We draw one random permutation of the corpus, give the first half to one modality and the second half to the
other, so no sample is seen on both sides and no correspondence between the two sets exists.
In the cross-dataset setting the two sides are drawn independently from two different corpora instead.
Every experiment is repeated with five seeds, drawn once from a generator seeded with $42$, and the tables and
figures report the mean and the standard deviation over those five runs.
The exceptions are the geometry scores and the text-to-image retrieval, which use one seed, and the fMRI
benchmark and the granularity experiments, which use ten seeds.
Embeddings are computed once and cached in \texttt{bfloat16}.
A language corpus stores several texts per item, for instance the five captions of an image or the three repeats
of a generative prompt, as a padded tensor.
Such a tensor is reduced by normalizing every row, averaging over the texts of an item while ignoring the
padding, and normalizing the result again.

\inparagraph{Retrieval metrics.}
The reported metric is FOSCTTM, the fraction of samples closer to a query than its true partner, averaged over
all queries. Given a paired validation set of $p$ samples $(\hat{\vx}_i,\hat{\vy}_i)$, $i=1,\dots,p$ and a function $f$ that maps from the source to the target space, it is
\begin{equation}
\text{FOSCTTM} = 
\frac{1}{p (p - 1)}
\sum_{i, j=1}^{p}
\mathbf{1}
\left[
\|f(\hat{\vx}_i) - \hat{\vy}_j\|_2
<
\|f(\hat{\vx}_i) - \hat{\vy}_i\|_2
\right].
\end{equation}
It is $0$ for a perfect alignment and $0.5$ for a random correspondence, and ties count half.
Retrieval is evaluated on the full paired validation set.
Both sides are randomly permuted before they reach the aligner, so the row order carries no information.

\inparagraph{Geometry metrics.}
Four measures of how similar the two embedding geometries are enter \cref{subsec:unpaired,subsec:other_domains}.
CKA is the linear centered kernel alignment of \citet{kornblith2019similarity} with the unbiased estimator of the
Hilbert-Schmidt independence criterion (HSIC).
Mutual $k$-NN is the measure of \citet{huh2024platonic}, the average overlap of the $k$ nearest neighbors of a
sample in the two spaces, with $k = 10$, an inner-product kernel and the sample itself excluded.
TSI and QSI~\citep{soares2026scalable} are the fractions of sampled triplets $(i,j,k)$ with $d(i,j) < d(i,k)$ and of sampled quadruplets
$(i,j,k,l)$ with $d(i,j) < d(k,l)$ whose ordering agrees in both spaces, each estimated from $10^5$ samples.
All four are computed on $10{,}000$ paired validation samples with each side centered.
A corpus with fewer than $10{,}000$ paired samples contributes all of them, without repetition.
Several of the modality pairs of \cref{subsec:other_domains} are far below that, the smallest being $123$ paired
slides for tissue against MRI, $149$ words for MEG against text and $168$ patients for tissue against CT.
Their geometry scores are therefore estimated from a few hundred points rather than ten thousand and carry much
more variance than the vision-language ones.
All three of those pairs sit at chance in \cref{tab:other_domains}, so the conclusion drawn from them is that
there is no alignment to find, which a noisy estimate supports as much as a precise one.

\inparagraph{Unpaired cross-modal alignment (\cref{subsec:unpaired}).}
The grid is seven vision models against three language models on four corpora.
The vision models are iBOT ViT-B/16 and iBOT Swin-T/14~\citep{zhou2022image}, DINOv2 ViT-B/14 and DINOv2 ViT-G/14~\citep{Oquab:2024:DIN}, Franca ViT-G/14~\citep{venkataramanan2026franca} trained on LAION~\citep{schuhmann2022laion} in both its class-token and its mean-pooled read-out,
and DINOv3 ViT-7B/16 at a resolution of $512$ pixels~\citep{simeoni2026dinov}.
All of them are used at a resolution of $224$ pixels except DINOv3, and all are mean pooled over the class,
register and patch tokens except the class-token variant of Franca.
The language models are all-mpnet-base-v2 from Sentence Transformers~\citep{Reimers:2019:SBE}, Qwen3-Embedding-8B~\citep{zhang2025qwen3}, and Qwen3-8B~\citep{yang2025qwen3} with the generative token pooling of \citet{wang2026the}.
Generative pooling prompts the model with \emph{Imagine what it would look like to see:} followed by the caption,
generates at most $128$ tokens, and averages the hidden states of the generated tokens.
The prompt is repeated three times and the three results are averaged.
The corpora are MS COCO 2014 train with $82{,}783$ images and up to seven captions each~\citep{chen2015microsoft},
Stanford Paragraph Captioning with $19{,}561$ images~\citep{krause2016paragraphs}, Densely Captioned Images with $7{,}805$ images in its
extended caption variant~\citep{urbanek2024picture}, and DOCCI with $14{,}847$ images~\citep{onoe2024docci}.
Validation always uses the $40{,}504$ images of MS COCO 2014 val~\citep{chen2015microsoft}.
The cross-dataset setting takes the images from MS COCO and the captions from Stanford Paragraph Captioning, so
the two sides describe different scenes.

\inparagraph{Zero-shot classification.}
The same aligners are evaluated on CIFAR-10~\citep{Krizhevsky:2009:LML}, CIFAR-100~\citep{Krizhevsky:2009:LML} and ImageNet-100~\citep{tian2020contrastive} by mapping the image embeddings into the language space and retrieving the nearest class prompt.
CIFAR-10 and CIFAR-100 use the same $18$ photo templates from CLIP~\citep{Radford:2021:LTV}, such as \emph{a photo of a \{\}} and \emph{a blurry
photo of the \{\}}, and the class embedding is the average over the templates.
ImageNet-100 uses no template and embeds the class name followed by its WordNet definition.
We report top-1 accuracy on CIFAR-10 and top-5 accuracy on CIFAR-100 and ImageNet-100.

\inparagraph{Few-pair alignment (\cref{subsec:few_pair}).}
This experiment uses MS COCO with DINOv2 ViT-B/14 against Qwen3-8B with generative pooling, and varies the number
of known pairs over $0$, $1$, $2$, $5$, $10$, $20$, $50$, $100$, $200$, $500$ and $1000$.
The pairs enter all three stages, as the linear term of the quadratic assignment problem in \cref{lem:cka_pairs}
and as an additional term in both Procrustes steps.
Every baseline receives exactly the same pairs and the same validation data.

\inparagraph{Text-to-image generation (\cref{subsec:text2img}).}
A caption is embedded, mapped into the image space by the aligner, and generated by a diffusion model conditioned on the image embedding.
We use a representation autoencoder (RAE, \citet{zheng2026diffusion}) as a diffusion model.
Instead of text conditioning, we condition the diffusion transformer on the mean-pooled DINOv2 ViT-B/14 embeddings with registers~\citep{darcet2024vision}, which is the encoder of the RAE.
We train the diffusion transformer on ImageNet-1K with center crop of $256$ pixels for $400$ epochs with a global batch size of $256$ on $4 \times$ A40 GPUs.
During generation, we sample images with $250$ steps and a guidance scale of $1.0$.
The resulting unnoised patch embeddings are then decoded into images by the pretrained RAE decoder.
The whole diffusion model is trained exclusively on image data and never sees a caption, so the language side enters
exclusively through the aligner.
The language models are all-mpnet-base-v2 and Contriever~\citep{izacard2022unsupervised}.
The aligner is fitted on MS COCO 2014 train and
the number of known pairs is varied over $0$ to $100$.
The eight captions shown are fixed across all grids, and one seed is used throughout, so that two grids differ
only in the aligner.
The linear baseline is the least-squares map from the same known pairs, obtained from the pseudo-inverse, with
the same preprocessing.
With zero pairs it is an arbitrary map and generates unrelated images.

\inparagraph{Text-to-image scores.}
We score the generated images on two sets of prompts.
The first set is the test split of CyclePrefDB-T2I~\citep{bahng2025cycle}, whose $380$ prompts summarize dense
captions of DCI photographs~\citep{urbanek2024picture}.
We take the photograph of each prompt from the test split of CyclePrefDB-I2T, which holds the same photographs in
the same order.
The second set has one caption for each of the $40{,}504$ images of the MS COCO 2014 validation split, namely the
caption with the lowest annotation id.
Every setting generates one image per prompt.
The initial noise depends only on the prompt, so all aligners start from the same noise.
As a model trained on paired data, we use Scale-RAE~\citep{tong2026scaling}, which combines
Qwen2.5-1.5B with a diffusion transformer of $2.4$B parameters and generates images of $224$ pixels.
We sample it as its own evaluation script does, with the prefix \emph{Generate an image of} and a guidance scale
of $1.0$.
Every image is scored against its prompt with four measures.
CLIPScore~\citep{hessel2021clipscore} is the rescaled cosine similarity between the CLIP ViT-L/14 embeddings of
the image and of the caption.
VQAScore~\citep{lin2024evaluating} is the probability that CLIP-FlanT5-XL assigns to the answer \emph{Yes} for the
question whether the image shows the caption.
TIFA~\citep{hu2023tifa} generates question and answer pairs from the caption with its open LLaMA-2~\citep{touvron2023llama} question generator and keeps only those that a
question answering model answers correctly from the caption alone.
The score is the fraction of the kept questions that a visual question answering model answers correctly on the
image.
We use the official implementation with its released LLaMA-2 question generator instead of GPT-3.5 and with
BLIP-large as the visual question answering model, which answers freely.
An answer that is not one of the choices counts as the closest choice under Sentence-BERT~\citep{Reimers:2019:SBE}.
CycleReward~\citep{bahng2025cycle} is a learned preference model with an arbitrary scale, so only differences
between methods are meaningful.
We use its Combo checkpoint.
We report the mean and the standard error over prompts, and TIFA leaves out the prompts for which no question
passes the filter.

\inparagraph{Domain-specific alignment (\cref{subsec:domain_specific}).}
Three benchmarks from three fields, each aligning two spaces that no shared encoder connects.
The first is Natural Questions (NQ)~\citep{kwiatkowski2019natural}, where two different sentence encoders embed the same $5{,}332{,}023$ passages.
The last $8192$ passages are the paired validation set and the two disjoint training halves take $250{,}000$
passages each.
The five encoders (granite~\citep{awasthy2025granite}, e5~\citep{wang2022text}, gte~\citep{li2023towards}, gtr~\citep{ni2022large}, stella~\citep{zhang2024jasper}) give ten ordered pairs.
The second is the PBMC benchmark~\citep{10xgenomics2021pbmc} of $2407$ cells measured with two assays, with gene
expression reduced to $50$ principal components and chromatin accessibility to $50$ topics.
We additionally report label transfer accuracy (LTA), the fraction of cells whose cell type is recovered from the
nearest neighbor in the other assay.
The third is the fMRI benchmark of the Natural Scenes Dataset (NSD)~\citep{allen2022massive}, in the setting of \citet{marcos2026platonic}, where a
per-subject encoder maps brain responses of eight subjects into a shared $128$-dimensional space.
The eight subjects give $56$ ordered pairs.

\inparagraph{Shared geometry across modalities (\cref{subsec:other_domains}).}
Eleven further modality pairs from the natural sciences, each with an independently trained encoder on either
side, are listed with their encoders and validation sizes in \cref{tab:other_domains_setup}.
The aligner is fitted on two disjoint halves of the unpaired data and evaluated on all paired samples.
The geometry scores are computed on the same validation rows as the retrieval scores, so the two axes of \cref{fig:other_domains} come from one run.

\begin{table}[t]
    \centering\tiny
    \begin{tabularx}{\textwidth}{@{}>{\raggedright\arraybackslash}X >{\raggedright\arraybackslash}X >{\raggedright\arraybackslash}X r@{}}
\toprule
Pair & Encoder 1 & Encoder 2 & Val. \\
\midrule
MLIP $\leftrightarrow$ MLIP (MP-20)~\citep{jain2013commentary} & \texttt{mace\_mp\_small-\allowbreak pca256-\allowbreak structure}\citep{batatia2025foundation} & \texttt{mace\_mp\_medium-\allowbreak pca256-\allowbreak structure}\citep{batatia2025foundation} & 12{,}000 \\
human $\leftrightarrow$ mouse scRNA~\citep{czi2025cellxgene} & \texttt{pca50} & \texttt{pca50} & 34{,}500 \\
scRNA $\leftrightarrow$ scATAC (SNARE-seq)~\citep{chen2019high} & \texttt{pca10} & \texttt{lsi19} & 1{,}047 \\
histology $\leftrightarrow$ expression (HEST CCRCC)~\citep{jaume2024hest} & \texttt{h\_optimus\_0}~\citep{saillard2024hoptimus} & \texttt{expression-pca64-ccrcc} & 74{,}220 \\
scRNA $\leftrightarrow$ ADT (CITE-seq)~\citep{stoeckius2017simultaneous} & \texttt{raw25} & \texttt{raw25} & 1{,}000 \\
tissue $\leftrightarrow$ CT (CPTAC)~\citep{edwards2015cptac} & \texttt{h\_optimus\_0}~\citep{saillard2024hoptimus} & \texttt{ct\_fm}~\citep{pai2025vision} & 168 \\
MEG $\leftrightarrow$ text (SpanishBCBL)~\citep{levy2025brain} & \texttt{meg\_words} & \texttt{mpnet-multi}~\citep{reimers2020making} & 149 \\
tissue $\leftrightarrow$ MRI (TCGA glioma)~\citep{bakas2017advancing} & \texttt{gigapath}~\citep{xu2024whole} & \texttt{dinov1\_vitb16-comp}~\citep{Caron:2021:EPS} & 123 \\
Cell Painting $\leftrightarrow$ L1000 (Rosetta)~\citep{haghighi2022high} & \texttt{morphology-pca512} & \texttt{expression-pca256} & 2{,}732 \\
MS/MS $\leftrightarrow$ molecules (MassSpecGym)~\citep{bushuiev2024massspecgym} & \texttt{dreams}~\citep{bushuiev2026self} & \texttt{ecfp4}~\citep{rogers2010extended} & 31{,}602 \\
galaxy $\leftrightarrow$ spectrum (DESI)~\citep{parker2024astroclip} & \texttt{astrodino}~\citep{parker2024astroclip} & \texttt{specformer}~\citep{parker2024astroclip} & 29{,}697 \\
\bottomrule
\end{tabularx}

    \caption{\textbf{The modality pairs of \cref{subsec:other_domains}.} We test our aligner on a variety of independently trained models. The last
    column gives the number of paired samples the alignment is evaluated on.}
    \label{tab:other_domains_setup}
\end{table}
\FloatBarrier
\section{Additional Evaluation Results}
\label{app:results}

\subsection{Unpaired Cross-Modal Alignment}
\label{app:unpaired}
\cref{fig:panel_coco_train2014,fig:panel_StanfordParagraphCaptioning,fig:panel_DenselyCaptionedImages,fig:panel_DOCCIDataset}
repeat the main figure for every dataset, MS COCO as in the main text and the three detailed captioning corpora.
The cross-dataset results are shown in \cref{fig:panel_coco_spc} in which the images come from MS COCO and the captions from Stanford Paragraph
Captioning, so that no correspondence between the two sets exists.
For all datasets, most vision and language models can be aligned well without any pairs and on every dataset our aligner outperforms the baselines mini-vec2vec and vec2vec on average.
In general, iBOT ViT-B/16 performs worse than the other vision models.
On all detailed captioning datasets, the generative token pooling performs substantially worse.
\cref{app:gen_token_pooling} shows that this method is mostly beneficial for short captions and hurts on long ones.
Mean pooling for Franca ViT-G/14 performs similar to taking the [CLS] token.
Finally, we don't observe a strong trend that bigger models can be aligned better, even though our study is not big enough to conclude on this point.

\begin{figure}[p]
    \centering
    \includegraphics{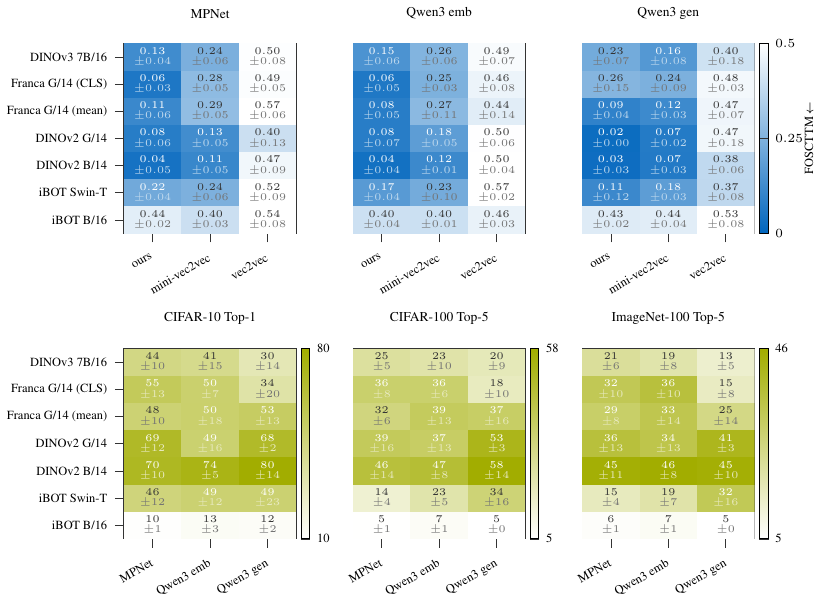} {} 

    \caption{\textbf{Unpaired alignment on MS COCO.} \textit{Top:} FOSCTTM across all vision-language method combinations (white: chance level; blue: better alignment). Our method substantially outperforms vec2vec and mini-vec2vec across most vision-language combinations. \textit{Bottom:} Zero-shot accuracy of our aligner (\%, white: chance level; green: better alignment). Despite being trained without paired data, our method reaches zero-shot accuracies well above chance for most model pairs.}
    \label{fig:panel_coco_train2014}
\end{figure}

\begin{figure}[p]
    \centering
    \includegraphics{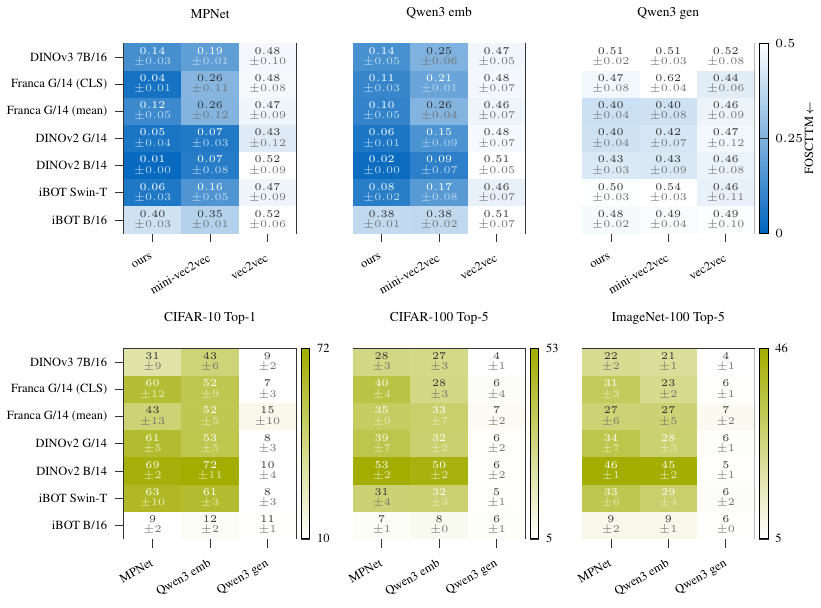} {} 

    \caption{\textbf{Unpaired alignment on SPC.} \textit{Top:} FOSCTTM across all vision-language method combinations (white: chance level; blue: better alignment). Our method substantially outperforms vec2vec and mini-vec2vec across most vision-language combinations. As observed across all detailed captioning datasets, none of the methods is able to reliably align Qwen3 with generative token pooling~\citep{wang2026the} on SPC. \textit{Bottom:} Zero-shot accuracy of our aligner (\%, white: chance level; green: better alignment). Despite being trained without paired data, our method reaches zero-shot accuracies well above chance for most model pairs.}
    \label{fig:panel_StanfordParagraphCaptioning}
\end{figure}

\begin{figure}[p]
    \centering
    \includegraphics{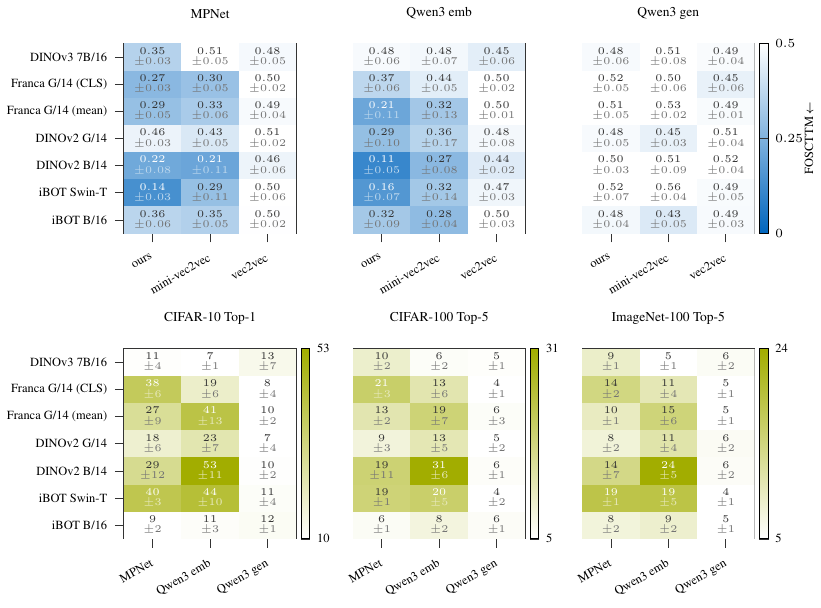} {} 

    \caption{\textbf{Unpaired alignment on DCI.} \textit{Top:} FOSCTTM across all vision-language method combinations (white: chance level; blue: better alignment). Our method substantially outperforms vec2vec and mini-vec2vec across most vision-language combinations. As observed across all detailed captioning datasets, none of the methods is able to reliably align Qwen3 with generative token pooling~\citep{wang2026the} on DCI. \textit{Bottom:} Zero-shot accuracy of our aligner (\%, white: chance level; green: better alignment). Despite being trained without paired data, our method reaches zero-shot accuracies well above chance for most model pairs, although lower than on MS COCO.}
    \label{fig:panel_DenselyCaptionedImages}
\end{figure}

\begin{figure}[p]
    \centering
    \includegraphics{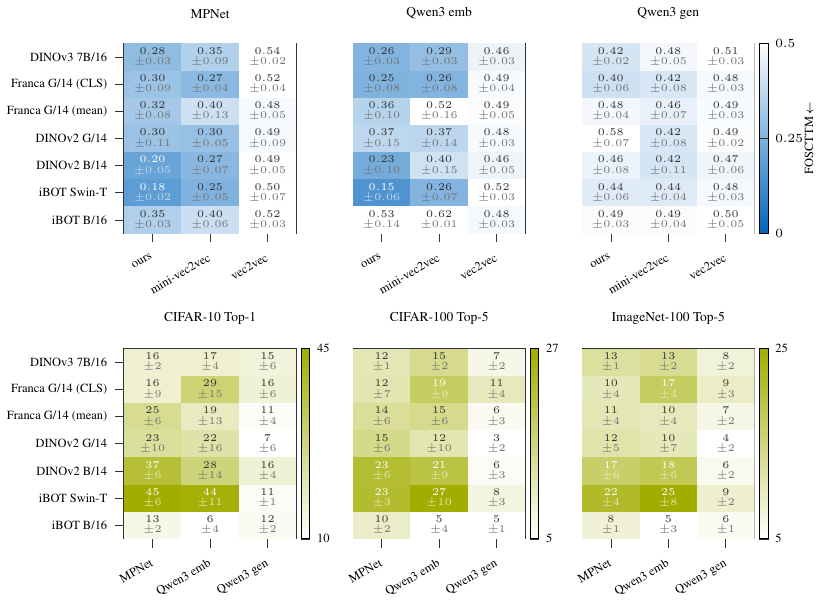} {} 

    \caption{\textbf{Unpaired alignment on DOCCI.} \textit{Top:} FOSCTTM across all vision-language method combinations (white: chance level; blue: better alignment). Our method substantially outperforms vec2vec and mini-vec2vec across most vision-language combinations. As observed across all detailed captioning datasets, none of the methods is able to reliably align Qwen3 with generative token pooling~\citep{wang2026the} on DOCCI. \textit{Bottom:} Zero-shot accuracy of our aligner (\%, white: chance level; green: better alignment). Despite being trained without paired data, our method reaches zero-shot accuracies well above chance for most model pairs, although lower than on MS COCO.}
    \label{fig:panel_DOCCIDataset}
\end{figure}

\begin{figure}[p]
    \centering
    \includegraphics{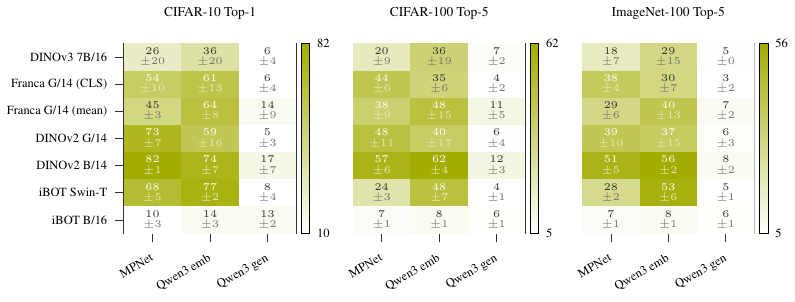} {} 

    \caption{\textbf{Zero-shot accuracy in the cross-dataset setting} Zero-shot accuracy of our aligner (\%, white: chance level; green: better alignment). We fit our aligner with images from MS COCO and captions from SPC. The FOSCTTM of this setting is the ours$^\dagger$ column of \cref{fig:unpaired_coco} in the main text. Our aligner reaches zero-shot accuracies comparable to the single-dataset setting for MPNet and Qwen3-Embedding-8B, while generative token pooling drops to chance.}
    \label{fig:panel_coco_spc}
\end{figure}

\inparagraph{The other geometric measures.}
\cref{fig:alignment_metrics} of the main text relates the alignment to CKA.
\cref{fig:alignment_metrics_full} adds the three other measures (mutual $k$-NN, TSI, and QSI) on the same model pairs and datasets and with the
same colors.
All four are predictive, but they differ in how tightly they follow the alignment and in how they rank the model pairs.
Pooled over the four corpora, CKA explains the most variance at $R^2 = 0.73$, followed by mutual $k$-NN at
$0.40$, TSI at $0.37$ and QSI at $0.18$.

\begin{figure}[t]
    \centering
    \begin{subfigure}[b]{0.49\textwidth}
        \centering
        \includegraphics{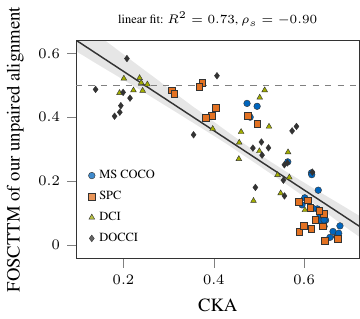} {} 
        \caption{CKA}
        \label{fig:alignment_metrics_cka}
    \end{subfigure}%
    \hfill
    \begin{subfigure}[b]{0.49\textwidth}
        \centering
        \includegraphics{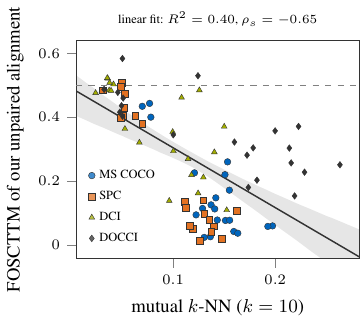} {} 
        \caption{mutual $k$-NN}
        \label{fig:alignment_metrics_mknn}
    \end{subfigure}

    \vspace{0.8em}
    \begin{subfigure}[b]{0.49\textwidth}
        \centering
        \includegraphics{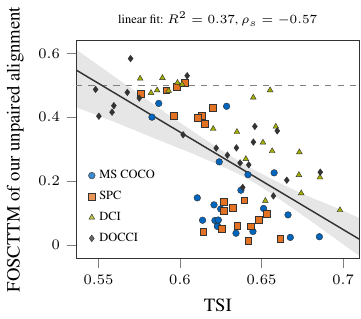} {} 
        \caption{TSI}
        \label{fig:alignment_metrics_tsi}
    \end{subfigure}%
    \hfill
    \begin{subfigure}[b]{0.49\textwidth}
        \centering
        \includegraphics{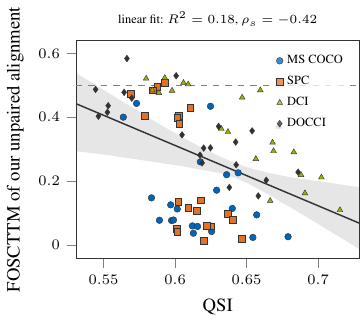} {} 
        \caption{QSI}
        \label{fig:alignment_metrics_qsi}
    \end{subfigure}

    \caption{\textbf{Shared geometry predicts alignment.} FOSCTTM of our unpaired alignment against four measures
    of the geometric similarity of the paired embeddings, one point per vision-language-dataset combination, colored by
    dataset, with a least-squares fit and its cluster-bootstrap confidence band. Dashed: the $0.5$ of a random
    correspondence. We observe that all four similarity measures predict the alignment even though CKA has the highest absolute correlation.}
    \label{fig:alignment_metrics_full}
\end{figure}

\FloatBarrier
\subsection{Text-to-Image Generation}
\label{app:text2img}
\cref{fig:text2img,fig:text2img_linear_mpnet,fig:text2img_contriever,fig:text2img_linear_contriever} show one
page per setting, our aligner and a linear map trained on the same pairs, each with MPNet and with Contriever as
the text encoder.
Every column is one caption and every row a number of known image-text pairs.
The diffusion model and the image decoder never see any text, so everything that changes between two grids is the
map from the language space into the image space.

The results do not depend on MS COCO captions in the training data of the text encoder.
MPNet is trained on a mixture of sentence pairs that includes MS COCO captions, so it could have seen the captions
our aligner is fitted on.
Contriever is trained on web text without MS COCO.
With Contriever, the grids show the same progression with the number of pairs, and our aligner scores higher than
the linear map on all four measures and for every number of pairs (\cref{tab:text2img_metrics}).
Without any pairs, our aligner with Contriever even scores higher on CyclePrefDB than the linear map with $10$ pairs with a TIFA of $0.419$ against $0.389$.
Contriever gives lower scores than MPNet overall, but the ordering between the two methods is the same for both
encoders.

We also measure how faithful the samples are on two sets of prompts, the $380$ test prompts of
CyclePrefDB~\citep{bahng2025cycle}, which summarize dense captions of photographs, and one caption for each of the
$40{,}504$ images of the MS COCO 2014 validation split.
Every setting generates one image per prompt, and all aligners start from the same noise for a given prompt.
We compare against the real image of each prompt and against Scale-RAE~\citep{tong2026scaling}, which conditions an
RAE diffusion model directly on a language model and is trained on tens of millions of image-text pairs.
Every image is scored against its prompt with the four measures of \cref{app:experimental_details}.
On CyclePrefDB, our aligner scores higher than the linear map on all four measures, for every number of pairs and
with both text encoders (\cref{tab:text2img_metrics}).
With MPNet and no pairs, our samples reach a CLIPScore of $0.466$, a VQAScore of $0.458$ and a TIFA of $0.552$,
which is above the linear map with $100$ pairs at $0.411$, $0.334$ and $0.507$.
Adding pairs improves our aligner only slightly, to $0.510$, $0.511$ and $0.615$ at $100$ pairs.
Contriever gives slightly lower scores than MPNet, but the ordering between the two methods stays the same.
Scale-RAE and the real images both reach a VQAScore and a TIFA above $0.82$.
The gap to both is expected, because only the aligner connects our diffusion model to text, using at most $100$
pairs.

On MS COCO, the scores show the same ordering as on CyclePrefDB (\cref{tab:text2img_metrics_coco}).
Our aligner again scores higher than the linear map on all four measures, for every number of pairs and with both
text encoders.
Pairs help our aligner more than on CyclePrefDB, and with MPNet its TIFA rises from $0.589$ without pairs to $0.686$
with $100$ pairs.
With Contriever and without pairs, our aligner scores below the linear map with $100$ pairs, e.g., with a TIFA of
$0.435$ against $0.521$.

\begin{table}[t]
    \centering\scriptsize
    \begin {tabular}{lllS[table-format=1.3(4)]S[table-format=1.3(4)]S[table-format=1.3(4)]S[table-format=-1.3(4)]}%
\toprule Text encoder&Method&Pairs&\multicolumn {1}{c}{CLIPScore}&\multicolumn {1}{c}{VQAScore}&\multicolumn {1}{c}{TIFA}&\multicolumn {1}{c}{CycleReward}\\\midrule %
--&GT image&--&0.614 +- 0.005&0.875 +- 0.006&0.840 +- 0.009&1.632 +- 0.038\\\midrule %
Qwen2.5-1.5B&Scale-RAE&--&0.660 +- 0.005&0.825 +- 0.008&0.821 +- 0.009&1.708 +- 0.034\\\midrule %
MPNet&ours&0&0.466 +- 0.005&0.458 +- 0.012&0.552 +- 0.014&-0.431 +- 0.051\\%
&ours&1&0.477 +- 0.005&0.486 +- 0.012&0.583 +- 0.013&-0.429 +- 0.050\\%
&ours&10&0.491 +- 0.005&0.491 +- 0.013&0.583 +- 0.013&-0.356 +- 0.048\\%
&ours&100&0.510 +- 0.005&0.511 +- 0.012&0.615 +- 0.013&-0.120 +- 0.051\\%
&linear&0&0.245 +- 0.005&0.207 +- 0.007&0.267 +- 0.012&-1.716 +- 0.016\\%
&linear&1&0.305 +- 0.006&0.211 +- 0.007&0.346 +- 0.012&-1.658 +- 0.017\\%
&linear&10&0.344 +- 0.006&0.259 +- 0.009&0.424 +- 0.014&-1.427 +- 0.029\\%
&linear&100&0.411 +- 0.005&0.334 +- 0.010&0.507 +- 0.013&-1.032 +- 0.040\\\midrule %
Contriever&ours&0&0.343 +- 0.006&0.253 +- 0.009&0.419 +- 0.014&-1.301 +- 0.035\\%
&ours&1&0.371 +- 0.006&0.290 +- 0.010&0.455 +- 0.014&-1.154 +- 0.038\\%
&ours&10&0.442 +- 0.005&0.375 +- 0.011&0.509 +- 0.014&-0.732 +- 0.044\\%
&ours&100&0.458 +- 0.005&0.413 +- 0.012&0.569 +- 0.014&-0.427 +- 0.052\\%
&linear&0&0.250 +- 0.005&0.227 +- 0.008&0.247 +- 0.011&-1.689 +- 0.014\\%
&linear&1&0.295 +- 0.006&0.202 +- 0.007&0.339 +- 0.012&-1.667 +- 0.017\\%
&linear&10&0.312 +- 0.006&0.230 +- 0.008&0.389 +- 0.013&-1.588 +- 0.021\\%
&linear&100&0.367 +- 0.006&0.259 +- 0.008&0.468 +- 0.014&-1.305 +- 0.035\\\bottomrule %
\end {tabular}

    \caption{\textbf{Faithfulness of the generated images on CyclePrefDB.}
    Our aligner outperforms the linear map on all four measures, for every number of pairs and with both text encoders.
    The real images and Scale-RAE~\citep{tong2026scaling} are included as reference points, but they are trained on tens of millions of image-text pairs and are not comparable to our aligner, which uses at most $100$ pairs.
    }
    \label{tab:text2img_metrics}
\end{table}

\begin{table}[t]
    \centering\scriptsize
    \begin {tabular}{lllS[table-format=1.4(5)]S[table-format=1.4(5)]S[table-format=1.4(5)]S[table-format=-1.4(5)]}%
\toprule Text encoder&Method&Pairs&\multicolumn {1}{c}{CLIPScore}&\multicolumn {1}{c}{VQAScore}&\multicolumn {1}{c}{TIFA}&\multicolumn {1}{c}{CycleReward}\\\midrule %
--&GT image&--&0.6389 +- 0.0005&0.8638 +- 0.0009&0.8695 +- 0.0008&0.2743 +- 0.0051\\\midrule %
Qwen2.5-1.5B&Scale-RAE&--&0.6456 +- 0.0004&0.8264 +- 0.0010&0.8528 +- 0.0009&0.5283 +- 0.0050\\\midrule %
MPNet&ours&0&0.5170 +- 0.0006&0.4284 +- 0.0014&0.5886 +- 0.0014&-1.3547 +- 0.0032\\%
&ours&1&0.5279 +- 0.0006&0.4644 +- 0.0014&0.6216 +- 0.0014&-1.2869 +- 0.0034\\%
&ours&10&0.5588 +- 0.0005&0.5197 +- 0.0014&0.6650 +- 0.0013&-1.1598 +- 0.0036\\%
&ours&100&0.5687 +- 0.0005&0.5412 +- 0.0013&0.6863 +- 0.0012&-1.1025 +- 0.0036\\%
&linear&0&0.2735 +- 0.0004&0.1498 +- 0.0005&0.2446 +- 0.0011&-1.8392 +- 0.0006\\%
&linear&1&0.3226 +- 0.0005&0.1929 +- 0.0007&0.2547 +- 0.0011&-1.8994 +- 0.0006\\%
&linear&10&0.3966 +- 0.0006&0.2491 +- 0.0010&0.3616 +- 0.0013&-1.7917 +- 0.0012\\%
&linear&100&0.4768 +- 0.0006&0.3686 +- 0.0012&0.5289 +- 0.0014&-1.5389 +- 0.0023\\\midrule %
Contriever&ours&0&0.4072 +- 0.0007&0.2685 +- 0.0011&0.4350 +- 0.0014&-1.6958 +- 0.0018\\%
&ours&1&0.4478 +- 0.0007&0.3190 +- 0.0012&0.4896 +- 0.0015&-1.5499 +- 0.0025\\%
&ours&10&0.5227 +- 0.0006&0.4452 +- 0.0013&0.6264 +- 0.0013&-1.3198 +- 0.0031\\%
&ours&100&0.5446 +- 0.0005&0.4789 +- 0.0013&0.6483 +- 0.0013&-1.2230 +- 0.0033\\%
&linear&0&0.2843 +- 0.0004&0.1559 +- 0.0005&0.2714 +- 0.0011&-1.8313 +- 0.0006\\%
&linear&1&0.3206 +- 0.0005&0.1925 +- 0.0007&0.2521 +- 0.0011&-1.9013 +- 0.0005\\%
&linear&10&0.3786 +- 0.0006&0.2384 +- 0.0009&0.3375 +- 0.0012&-1.8200 +- 0.0011\\%
&linear&100&0.4791 +- 0.0006&0.3624 +- 0.0012&0.5211 +- 0.0014&-1.5550 +- 0.0022\\\bottomrule %
\end {tabular}

    \caption{\textbf{Faithfulness of the generated images on MS COCO.}
    Our aligner outperforms the linear map on all four measures, for every number of pairs and with both text encoders.
    As in \cref{tab:text2img_metrics}, the real images and Scale-RAE are reference points and are not comparable to our
    aligner, which uses at most $100$ pairs.
    }
    \label{tab:text2img_metrics_coco}
\end{table}

\begin{figure}[p]
    \centering
    \vspace*{\fill}
    \includegraphics{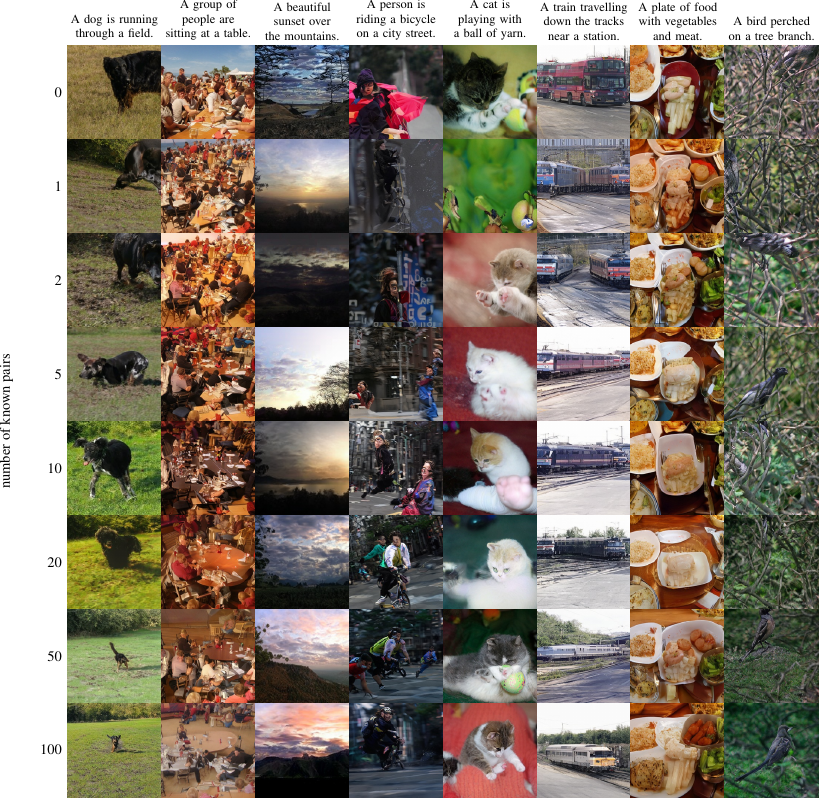} {} 
    \vspace*{\fill}
    \caption{\textbf{Unpaired text-to-image generation with our method and MPNet.} 
    We show the generated images for 8 sample captions with our aligner using DINOv2 ViT-B/14 and MPNet.
    Without pairs, our aligner can already produce images of the general semantic class and scene.
    Additional pairs improve fine-grained details.
    }
    \label{fig:text2img}
\end{figure}

\begin{figure}[p]
    \centering
    \vspace*{\fill}
    \includegraphics{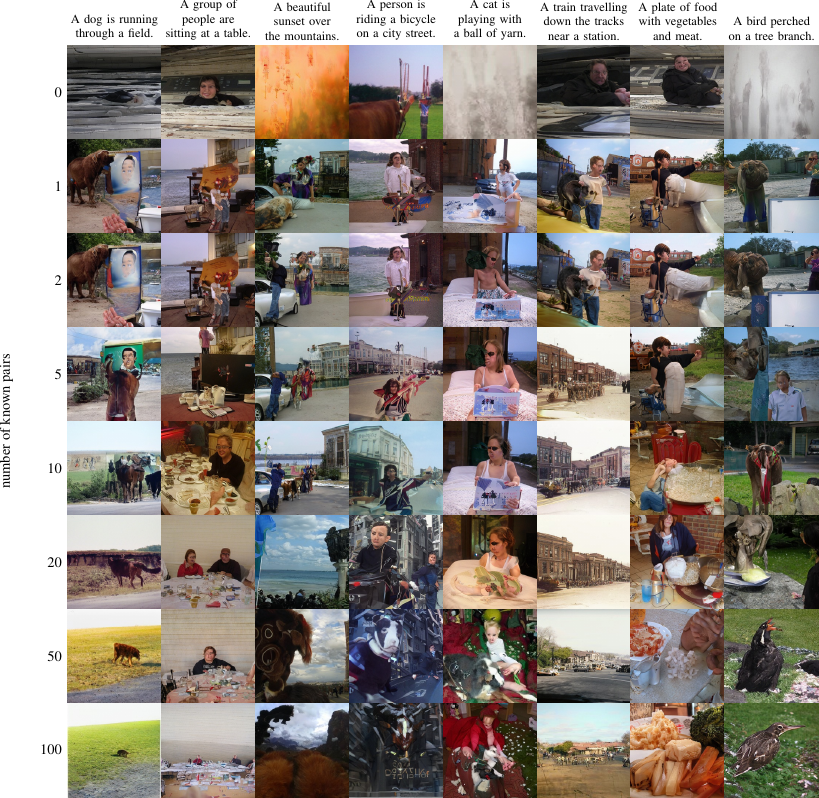} {} 
    \vspace*{\fill}
    \caption{\textbf{Unpaired text-to-image generation with a linear map and MPNet.} 
    We show the generated images for 8 sample captions with a linear map fitted on the pairs using DINOv2 ViT-B/14 and MPNet.
    Without pairs, the linear map is random. As more pairs are used, the faithfulness of the generation improves.}
    \label{fig:text2img_linear_mpnet}
\end{figure}

\begin{figure}[p]
    \centering
    \vspace*{\fill}
    \includegraphics{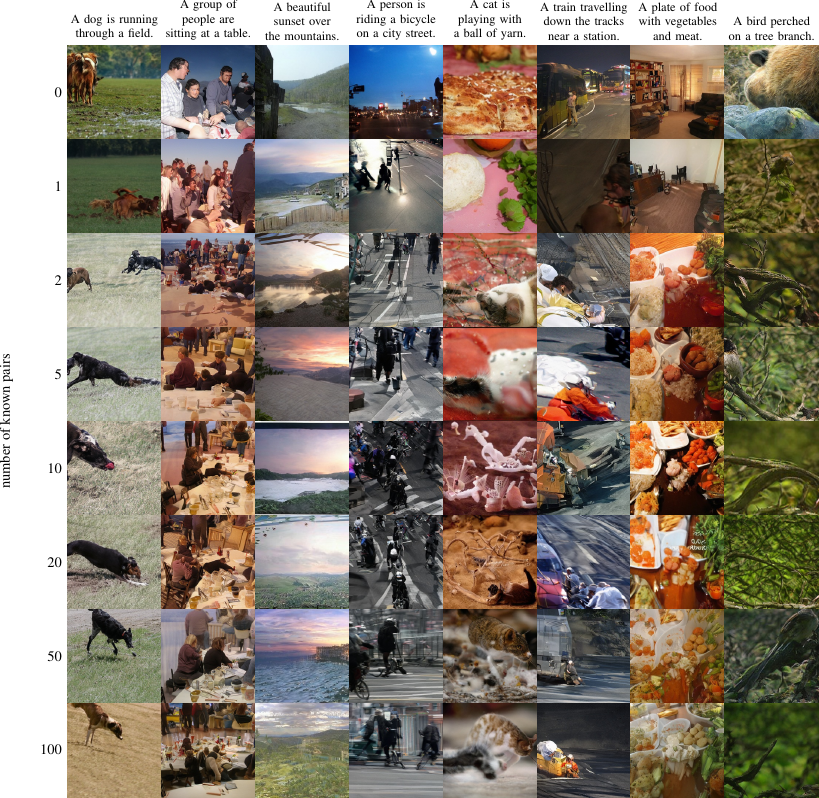} {} 
    \vspace*{\fill}
    \caption{\textbf{Unpaired text-to-image generation with our method and Contriever.} 
    We show the generated images for 8 sample captions with our aligner using DINOv2 ViT-B/14 and Contriever.
    Without pairs, our aligner can already produce images of the general semantic class and scene.
    Additional pairs improve fine-grained details.}
    \label{fig:text2img_contriever}
\end{figure}

\begin{figure}[p]
    \centering
    \vspace*{\fill}
    \includegraphics{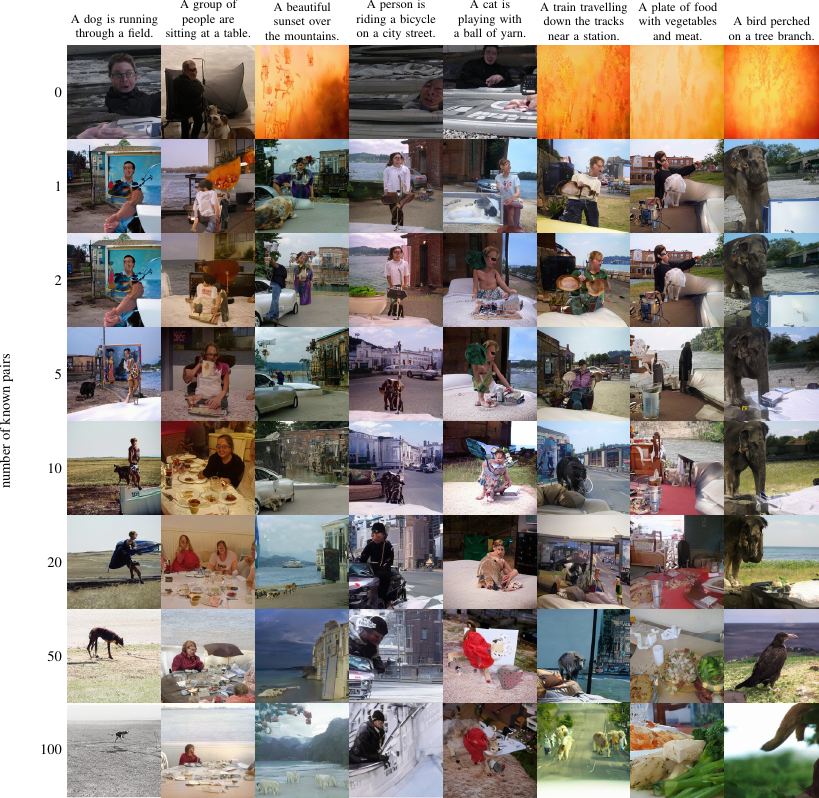} {} 
    \vspace*{\fill}
    \caption{\textbf{Unpaired text-to-image generation with a linear map and Contriever.} 
    We show the generated images for 8 sample captions with a linear map fitted on the pairs using DINOv2 ViT-B/14 and Contriever.
    Without pairs, the linear map is random. As more pairs are used, the faithfulness of the generation improves.}
    \label{fig:text2img_linear_contriever}
\end{figure}

\FloatBarrier
\subsection{Unpaired Domain-Specific Alignment}
\label{app:domain_specific}
\begin{table}[!t]
    \centering\scriptsize
    \begin {tabular}{lS[table-format=1.4]S[table-format=4.1]S[table-format=1.4]S[table-format=2.1]S[table-format=1.4]S[table-format=1.1]}%
\toprule & \multicolumn {2}{c}{vec2vec} & \multicolumn {2}{c}{mini-vec2vec} & \multicolumn {2}{c}{ours} \\ \cmidrule (lr){2-3}\cmidrule (lr){4-5}\cmidrule (lr){6-7}\\Pair&\multicolumn {1}{c}{FOSCTTM $\downarrow $}&\multicolumn {1}{c}{mean rank$\downarrow $}&\multicolumn {1}{c}{FOSCTTM $\downarrow $}&\multicolumn {1}{c}{mean rank$\downarrow $}&\multicolumn {1}{c}{FOSCTTM $\downarrow $}&\multicolumn {1}{c}{mean rank$\downarrow $}\\\midrule %
granite \ensuremath {\to } e5&0.5109&4185.7&\bfseries 0.0000&\bfseries 1.0&\bfseries 0.0000&1.1\\%
granite \ensuremath {\to } gte&0.4209&3448.6&\bfseries 0.0000&\bfseries 1.1&\bfseries 0.0000&1.2\\%
granite \ensuremath {\to } gtr&0.4070&3335.1&\bfseries 0.0000&\bfseries 1.2&\bfseries 0.0000&1.3\\%
granite \ensuremath {\to } stella&0.4875&3994.4&\bfseries 0.0000&\bfseries 1.1&\bfseries 0.0000&1.2\\%
gte \ensuremath {\to } e5&0.2861&2344.5&\bfseries 0.0000&\bfseries 1.1&\bfseries 0.0000&\bfseries 1.1\\%
gtr \ensuremath {\to } e5&0.4882&4000.0&\bfseries 0.0001&\bfseries 1.7&0.0002&2.2\\%
gtr \ensuremath {\to } gte&0.5288&4332.5&0.0024&20.6&\bfseries 0.0000&\bfseries 1.1\\%
gtr \ensuremath {\to } stella&0.5158&4225.7&\bfseries 0.0000&\bfseries 1.0&\bfseries 0.0000&\bfseries 1.0\\%
stella \ensuremath {\to } e5&0.3098&2538.2&\bfseries 0.0000&\bfseries 1.1&\bfseries 0.0000&\bfseries 1.1\\%
stella \ensuremath {\to } gte&0.1394&1142.7&\bfseries 0.0000&\bfseries 1.0&\bfseries 0.0000&\bfseries 1.0\\\bottomrule %
\end {tabular}

    \caption{\textbf{Per-pair results on NQ.} We report the FOSCTTM and mean rank of every encoder pair evaluated on $8{,}192$ validation samples. On average, mini-vec2vec and our aligner achieve nearly perfect alignment on all model pairs. Vec2vec is unstable and does not consistently demonstrate significant performance.}
    \label{tab:domain_nq_full}
\end{table}

\begin{table}[!t]
    \centering\scriptsize
    \begin {tabular}{lS[table-format=1.3]@{\,\ensuremath {\pm }\,}S[table-format=1.3]S[table-format=1.3]@{\,\ensuremath {\pm }\,}S[table-format=1.3]}%
\toprule Method&\multicolumn {2}{c}{FOSCTTM $\downarrow $}&\multicolumn {2}{c}{LTA$\uparrow $}\\\midrule %
SCOT+&0.121&0.027&0.919&0.015\\%
ours&\protect \bfseries  0.089&\protect \bfseries  0.003&\protect \bfseries  0.942&\protect \bfseries  0.019\\\bottomrule %
\end {tabular}

    \caption{\textbf{Result on PBMC.} We evaluate RNA $\to$ ATAC with FOSCTTM and label transfer accuracy (LTA). Our aligner outperforms the domain-specific aligner SCOT+~\citep{baker2026scot+} in FOSCTTM and label transfer accuracy.}
    \label{tab:domain_pbmc_full}
\end{table}

\cref{fig:domain_specific} of the main text averages every benchmark.
The tables below keep each encoder, omics and subject pair on its own row and give one column group per method,
the layout \citet{jha2026harnessing} and \citet{dar2025mini} use for these benchmarks.
We report the average FOSCTTM and the domain-specific metrics for every pair over the seeds.
Bold marks the best method of each row and metric, compared at the precision shown, so that two methods whose
numbers agree to the printed digits are both marked.

\inparagraph{Language.}
On the ten ordered encoder pairs of Natural Questions in \cref{tab:domain_nq_full}, vec2vec stays at or near the
chance value of $0.5$ on five of ten pairs and reaches $0.14$ at best.
Our aligner reaches a FOSCTTM of $0.0000$ on nine of ten pairs and $0.0002$ on the tenth, and mini-vec2vec
reaches $0.0000$ on eight.
What separates the two is not the typical case but the worst one.
On \texttt{gtr} to \texttt{gte} mini-vec2vec gives $0.0024 \pm 0.0045$ with a mean rank of $20.6 \pm 36.5$, while
ours gives $0.0000$ with a mean rank of $1.1 \pm 0.0$ on the same pair.
A spread that large across seeds matters here, because without any pairs there is no signal that could be used to
select the good run.

\inparagraph{Biology.}
On the PBMC benchmark in \cref{tab:domain_pbmc_full} our aligner reaches $0.089 \pm 0.003$ against
$0.121 \pm 0.027$ for SCOT+, which was designed for this kind of data.
The label transfer accuracy (LTA) tells the same story, $94.2\%$ against $91.9\%$ for the nearest cross-modal neighbor.
The standard deviation is again the clearer difference, ours being an order of magnitude smaller.

\inparagraph{Neuroscience.}
The fMRI benchmark in \cref{tab:domain_fmri_full} has $56$ ordered subject pairs.
Both methods use the same per-subject encoders.
Averaged over them our aligner reaches $0.033 \pm 0.025$ against $0.035 \pm 0.041$ for the encoders of
\citet{marcos2026platonic}, which was designed for this dataset.
The two are within each other's spread on the average, and the difference is again that ours varies less across
pairs and seeds.

\begin{table}[p]
    \centering\scriptsize
    \begin {tabular}{lS[table-format=1.3]S[table-format=3.1]S[table-format=1.3]S[table-format=2.1]}%
\toprule & \multicolumn {2}{c}{platonic brain} & \multicolumn {2}{c}{ours} \\ \cmidrule (lr){2-3}\cmidrule (lr){4-5}\\Pair&\multicolumn {1}{c}{FOSCTTM $\downarrow $}&\multicolumn {1}{c}{mean rank$\downarrow $}&\multicolumn {1}{c}{FOSCTTM $\downarrow $}&\multicolumn {1}{c}{mean rank$\downarrow $}\\\midrule %
sub01 \ensuremath {\to } sub02&0.040&36.9&\bfseries 0.005&\bfseries 5.2\\%
sub01 \ensuremath {\to } sub03&\bfseries 0.038&\bfseries 35.8&0.042&38.7\\%
sub01 \ensuremath {\to } sub04&0.073&67.5&\bfseries 0.040&\bfseries 37.1\\%
sub01 \ensuremath {\to } sub05&0.047&43.3&\bfseries 0.046&\bfseries 42.9\\%
sub01 \ensuremath {\to } sub06&0.030&28.0&\bfseries 0.026&\bfseries 25.0\\%
sub01 \ensuremath {\to } sub07&0.019&18.0&\bfseries 0.017&\bfseries 16.2\\%
sub01 \ensuremath {\to } sub08&\bfseries 0.042&\bfseries 39.4&0.056&52.1\\%
sub02 \ensuremath {\to } sub01&0.030&28.3&\bfseries 0.010&\bfseries 9.7\\%
sub02 \ensuremath {\to } sub03&0.070&64.3&\bfseries 0.043&\bfseries 39.8\\%
sub02 \ensuremath {\to } sub04&0.043&39.7&\bfseries 0.035&\bfseries 32.4\\%
sub02 \ensuremath {\to } sub05&0.126&114.8&\bfseries 0.064&\bfseries 58.7\\%
sub02 \ensuremath {\to } sub06&\bfseries 0.048&\bfseries 44.3&0.067&61.9\\%
sub02 \ensuremath {\to } sub07&0.044&40.8&\bfseries 0.029&\bfseries 27.0\\%
sub02 \ensuremath {\to } sub08&0.037&35.0&\bfseries 0.031&\bfseries 28.7\\%
sub03 \ensuremath {\to } sub01&\bfseries 0.043&\bfseries 39.8&0.054&49.6\\%
sub03 \ensuremath {\to } sub02&\bfseries 0.038&\bfseries 35.5&0.041&37.8\\%
sub03 \ensuremath {\to } sub04&\bfseries 0.013&\bfseries 13.0&0.034&32.3\\%
sub03 \ensuremath {\to } sub05&\bfseries 0.015&\bfseries 15.0&0.022&20.8\\%
sub03 \ensuremath {\to } sub06&\bfseries 0.018&\bfseries 16.9&0.019&17.8\\%
sub03 \ensuremath {\to } sub07&\bfseries 0.008&\bfseries 8.7&0.029&27.6\\%
sub03 \ensuremath {\to } sub08&0.049&45.6&\bfseries 0.021&\bfseries 20.0\\%
sub04 \ensuremath {\to } sub01&\bfseries 0.054&\bfseries 50.3&0.068&62.4\\%
sub04 \ensuremath {\to } sub02&0.070&64.7&\bfseries 0.044&\bfseries 41.0\\%
sub04 \ensuremath {\to } sub03&\bfseries 0.022&\bfseries 21.1&0.031&29.1\\%
sub04 \ensuremath {\to } sub05&\bfseries 0.012&\bfseries 12.2&0.014&13.7\\%
sub04 \ensuremath {\to } sub06&\bfseries 0.031&\bfseries 29.2&0.046&42.2\\%
sub04 \ensuremath {\to } sub07&\bfseries 0.031&\bfseries 28.7&0.034&31.7\\%
sub04 \ensuremath {\to } sub08&0.020&19.0&\bfseries 0.016&\bfseries 15.1\\%
sub05 \ensuremath {\to } sub01&\bfseries 0.034&\bfseries 31.6&0.038&35.3\\%
sub05 \ensuremath {\to } sub02&0.073&67.3&\bfseries 0.059&\bfseries 54.6\\%
sub05 \ensuremath {\to } sub03&\bfseries 0.017&\bfseries 16.0&0.028&26.3\\%
sub05 \ensuremath {\to } sub04&\bfseries 0.007&\bfseries 7.3&0.018&17.6\\%
sub05 \ensuremath {\to } sub06&0.014&13.3&\bfseries 0.013&\bfseries 13.0\\%
sub05 \ensuremath {\to } sub07&\bfseries 0.015&\bfseries 14.9&0.020&19.3\\%
sub05 \ensuremath {\to } sub08&0.051&47.1&\bfseries 0.026&\bfseries 24.1\\%
sub06 \ensuremath {\to } sub01&\bfseries 0.020&\bfseries 19.5&0.043&40.3\\%
sub06 \ensuremath {\to } sub02&0.075&68.7&\bfseries 0.060&\bfseries 55.2\\%
sub06 \ensuremath {\to } sub03&\bfseries 0.019&\bfseries 18.0&\bfseries 0.019&18.1\\%
sub06 \ensuremath {\to } sub04&0.035&32.4&\bfseries 0.032&\bfseries 30.1\\%
sub06 \ensuremath {\to } sub05&0.014&13.4&\bfseries 0.013&\bfseries 12.7\\%
sub06 \ensuremath {\to } sub07&0.011&11.4&\bfseries 0.007&\bfseries 7.2\\%
sub06 \ensuremath {\to } sub08&0.058&53.3&\bfseries 0.033&\bfseries 30.5\\%
sub07 \ensuremath {\to } sub01&\bfseries 0.009&\bfseries 9.3&0.014&13.4\\%
sub07 \ensuremath {\to } sub02&\bfseries 0.020&\bfseries 19.0&0.037&34.5\\%
sub07 \ensuremath {\to } sub03&\bfseries 0.012&\bfseries 11.4&0.033&30.8\\%
sub07 \ensuremath {\to } sub04&\bfseries 0.027&\bfseries 25.2&0.054&49.8\\%
sub07 \ensuremath {\to } sub05&0.024&23.0&\bfseries 0.013&\bfseries 13.2\\%
sub07 \ensuremath {\to } sub06&0.027&25.2&\bfseries 0.008&\bfseries 7.9\\%
sub07 \ensuremath {\to } sub08&\bfseries 0.045&\bfseries 41.6&0.046&42.8\\%
sub08 \ensuremath {\to } sub01&\bfseries 0.040&\bfseries 37.1&0.052&48.0\\%
sub08 \ensuremath {\to } sub02&\bfseries 0.037&\bfseries 34.5&0.047&43.7\\%
sub08 \ensuremath {\to } sub03&0.040&37.6&\bfseries 0.028&\bfseries 26.0\\%
sub08 \ensuremath {\to } sub04&\bfseries 0.013&\bfseries 12.7&0.016&15.8\\%
sub08 \ensuremath {\to } sub05&\bfseries 0.028&\bfseries 26.4&\bfseries 0.028&26.5\\%
sub08 \ensuremath {\to } sub06&0.038&35.8&\bfseries 0.036&\bfseries 33.8\\%
sub08 \ensuremath {\to } sub07&0.066&60.6&\bfseries 0.057&\bfseries 52.5\\\bottomrule %
\end {tabular}

    \caption{\textbf{Per-pair results on NSD fMRI.} We report the FOSCTTM score and the average rank of each pair of subjects, averaged over ten random seeds. On average, both methods perform similarly, even though our method is more stable and produces more consistent results.}
    \label{tab:domain_fmri_full}
\end{table}

\FloatBarrier
\subsection{Shared Geometry Across Modalities}
\label{app:other_domains}
\cref{tab:other_domains} gives the per-pair numbers behind \cref{fig:other_domains}.
We observe that the fourteen pairs fall into three groups.
\cref{fig:domain_spectrum} shows a visualization of what the three regimes look like.

\inparagraph{Well aligned.}
Four pairs are aligned almost exactly, Natural Questions at $0.000$, the two interatomic potentials on MP-20 at
$0.001$, the fMRI subject pairs at $0.033$, and PBMC at $0.089$.
SNARE-seq at $0.209$ recover the main cell populations.
All five have a CKA above $0.6$.
Two of them share an encoder family or a measurement device, which is the easy case, but PBMC and SNARE-seq
relate two different assays of the same cells and Natural Questions relates two independently trained text
encoders, so a shared architecture is not what makes them work.

\inparagraph{Partially aligned.}
Three pairs land between $0.22$ and $0.37$, well below chance but far from solved, namely CITE-seq at $0.225$,
human against mouse single-cell RNA at $0.338$ and histology against expression at $0.369$.
These pairs share some structure, but the geometry is not enough to identify individual samples.
CITE-seq has a high CKA of $0.96$, but its FOSCTTM varies strongly over the seeds, with a standard deviation of $0.13$.

\inparagraph{At chance.}
Six pairs are close to chance, Cell Painting against L1000 at $0.483$, mass spectra against molecules at $0.475$,
tissue against MRI at $0.484$, MEG against text at $0.498$, tissue against CT at $0.487$ and galaxy images against
spectra at $0.574$.
Their CKA is at most $0.25$ in every case, so little shared geometry is available.

\inparagraph{Which measure predicts best.}
Across these domains CKA and TSI both track the alignment, while mutual $k$-NN is a weaker predictor.
SNARE-seq is the clearest counterexample, the fifth best aligned pair of the fourteen at $0.209$ and yet a
mutual $k$-NN of $0.043$, lower than pairs that are at chance.
The different datasets considered here contain different numbers of samples.
Mutual $k$-NN measures the overlap of the $k$ nearest neighbors which is sensitive to the density of the point clouds, and the datasets here differ in size by two orders of magnitude.
CKA and TSI compare global similarity structure and are therefore more stable across different sample sizes, which may explain why they predict the alignment better in this setting.

\inparagraph{When no aligner is needed at all.}
The pairs above relate modalities that no single encoder covers.
We also study the opposite case on five pairs taken from the CycleGAN literature, where one encoder can be applied
to both sides, namely MR against CT and CBCT against CT on SynthRAD2023~\citep{thummerer2023synthrad2023}, RGB against thermal on LLVIP~\citep{jia2021llvip}, SAR against
optical on SEN12MS-CR~\citep{sen12mscr} and one speaker against another on CMU Arctic~\citep{kominek2004cmu}.
The four image pairs are embedded with self-supervised vision transformers and the speaker pair with a
self-supervised speech model, one encoder per pair as listed in \cref{tab:cyclegan_pairs}.
Every pair uses the same encoder on both modalities, so the two embedding spaces share a basis by construction and
the identity map is a meaningful baseline.
We report the identity map, the identity map after subtracting each side's training mean and normalizing the rows,
and our aligner, each over the same five seeds and the same $1024$ sample validation slice.
\cref{tab:cyclegan_pairs} gives the numbers.
Our aligner recovers a coarse map on every pair, at $0.020$ for RGB against thermal, $0.054$ for CBCT against CT,
$0.217$ for SAR against optical, $0.265$ for the two speakers and $0.253$ for MR against CT, although it uses no
identity bias, doesn't assume the same encoders, and never sees a pair.
Centering and normalizing alone is better on four of the five pairs and within $0.001$ on RGB against thermal.
On CMU Arctic it recovers the correspondence exactly at $0.000$ because both speakers read the same sentences and one encoder maps them to nearly
the same point.
This is a stronger statement than the one the rest of the paper makes.
Where a single encoder covers both modalities, they can be aligned coarsely across modalities without
fitting anything at all.
Such an alignment would allow architectures that place a generative model in each space and translate between them
without a GAN, in the same way as our text-to-image pipeline in \cref{subsec:text2img}.
Our setting is not designed for that use, since we embed one token per image and the map is not pixel-wise by
construction, so we leave this direction to future work.
The same result also marks the limit of the comparison.
When the same model can be used on both sides, centering and normalizing is the method of choice, and our aligner
is meant for the pairs where no such model exists, such as vision and language.

\begin{table}[t]
    \centering\scriptsize
    \begin {tabularx}{\linewidth }{>{\raggedright \arraybackslash \hspace {0pt}}XlS[table-format=1.3]S[table-format=1.3]S[table-format=1.3]S[table-format=1.3]S[table-format=1.3]}%
\toprule & & \multicolumn {3}{c}{FOSCTTM $\downarrow $ per method} & \multicolumn {2}{c}{shared geometry} \\ \cmidrule (lr){3-5}\cmidrule (lr){6-7}\\Setting&Encoder (both sides)&\multicolumn {1}{c}{identity}&\multicolumn {1}{c}{centred+norm}&\multicolumn {1}{c}{ours}&\multicolumn {1}{c}{CKA $\uparrow $}&\multicolumn {1}{c}{mutual $k$-NN $\uparrow $}\\\midrule %
MR \ensuremath {\leftrightarrow } CT (\allowbreak {}brain)&DINOv1 Vi\allowbreak {}T-B/\allowbreak {}16&0.178&\bfseries 0.108&0.253&0.773&0.124\\%
CBCT \ensuremath {\leftrightarrow } CT (\allowbreak {}brain)&DINOv2 Vi\allowbreak {}T-B/\allowbreak {}14&0.156&\bfseries 0.048&0.054&0.710&0.226\\%
RGB \ensuremath {\leftrightarrow } thermal&DINOv2 Vi\allowbreak {}T-L/\allowbreak {}14&0.053&0.021&\bfseries 0.020&0.894&0.167\\%
SAR \ensuremath {\leftrightarrow } optical&DINOv2 Vi\allowbreak {}T-L/\allowbreak {}14&0.260&\bfseries 0.177&0.217&0.624&0.092\\%
speaker clb \ensuremath {\leftrightarrow } rms&Wav\allowbreak {}LM-Large&\bfseries 0.000&\bfseries 0.000&0.265&0.855&0.474\\\bottomrule %
\end {tabularx}

    \caption{\textbf{Alignment is not required for some domains.} We consider multiple modality pairs from the CycleGAN~\citep{zhu2017unpaired} literature, where one encoder is applied to both modalities.
    We use DINOv1~\citep{Caron:2021:EPS}, DINOv2~\citep{Oquab:2024:DIN}, and WavLM~\citep{chen2022wavlm} as the encoders.
    The identity column applies no map, and the centered and normalized column subtracts the training mean from each side and normalizes the rows.
    Since they come from the same encoder, we observe that the embedding spaces are directly aligned, and centering and normalization help.
    This emphasizes that domains that can use the same encoders, such as RGB and thermal, are directly compatible in their embedding spaces when general self-supervised models are used.
    Although our aligner doesn't assume a shared embedding space, it still recovers a coarse alignment for these modalities.
    }
    \label{tab:cyclegan_pairs}
\end{table}

\begin{table}[t]
    \centering\scriptsize
    \begin {tabularx}{\linewidth }{l>{\raggedright \arraybackslash \hspace {0pt}}XS[table-format=1.2(3)]S[table-format=1.2(3)]S[table-format=1.2(3)]S[table-format=1.2(3)]}%
\toprule Domain&Pair&\multicolumn {1}{c}{FOSCTTM $\downarrow $}&\multicolumn {1}{c}{CKA $\uparrow $}&\multicolumn {1}{c}{mutual $k$-NN $\uparrow $}&\multicolumn {1}{c}{TSI $\uparrow $}\\\midrule %
astronomy&galaxy \ensuremath {\leftrightarrow } spectrum (\allowbreak {}DESI)&\multicolumn {1}{c}{0.57 \ensuremath {\pm } 0.00}&\multicolumn {1}{c}{0.18 \ensuremath {\pm } 0.00}&\multicolumn {1}{c}{0.01 \ensuremath {\pm } 0.00}&\multicolumn {1}{c}{0.55 \ensuremath {\pm } 0.00}\\%
biology&Cell Painting \ensuremath {\leftrightarrow } L1000 (\allowbreak {}Rosetta)&0.48 +- 0.10&\multicolumn {1}{c}{0.25 \ensuremath {\pm } 0.00}&\multicolumn {1}{c}{0.18 \ensuremath {\pm } 0.00}&\multicolumn {1}{c}{0.58 \ensuremath {\pm } 0.00}\\%
biology&human \ensuremath {\leftrightarrow } mouse sc\allowbreak {}RNA&0.34 +- 0.04&\multicolumn {1}{c}{0.32 \ensuremath {\pm } 0.00}&\multicolumn {1}{c}{0.10 \ensuremath {\pm } 0.00}&\multicolumn {1}{c}{0.60 \ensuremath {\pm } 0.00}\\%
biology&sc\allowbreak {}RNA \ensuremath {\leftrightarrow } ADT (\allowbreak {}CITE-seq)&0.22 +- 0.13&\multicolumn {1}{c}{0.96 \ensuremath {\pm } 0.00}&\multicolumn {1}{c}{0.14 \ensuremath {\pm } 0.00}&\multicolumn {1}{c}{0.67 \ensuremath {\pm } 0.00}\\%
biology&sc\allowbreak {}RNA \ensuremath {\leftrightarrow } sc\allowbreak {}ATAC (\allowbreak {}PBMC)&\multicolumn {1}{c}{0.09 \ensuremath {\pm } 0.00}&\multicolumn {1}{c}{0.90 \ensuremath {\pm } 0.00}&\multicolumn {1}{c}{0.06 \ensuremath {\pm } 0.00}&\multicolumn {1}{c}{0.76 \ensuremath {\pm } 0.00}\\%
biology&sc\allowbreak {}RNA \ensuremath {\leftrightarrow } sc\allowbreak {}ATAC (\allowbreak {}SNARE-seq)&0.21 +- 0.03&\multicolumn {1}{c}{0.86 \ensuremath {\pm } 0.00}&\multicolumn {1}{c}{0.04 \ensuremath {\pm } 0.00}&\multicolumn {1}{c}{0.74 \ensuremath {\pm } 0.00}\\%
chemistry&MS/\allowbreak {}MS \ensuremath {\leftrightarrow } molecules (\allowbreak {}Mass\allowbreak {}Spec\allowbreak {}Gym)&0.47 +- 0.02&\multicolumn {1}{c}{0.19 \ensuremath {\pm } 0.00}&\multicolumn {1}{c}{0.11 \ensuremath {\pm } 0.00}&\multicolumn {1}{c}{0.54 \ensuremath {\pm } 0.00}\\%
language&NQ text \ensuremath {\leftrightarrow } text&\multicolumn {1}{c}{0.00 \ensuremath {\pm } 0.00}&0.66 +- 0.09&0.32 +- 0.09&0.70 +- 0.03\\%
materials science&MLIP \ensuremath {\leftrightarrow } MLIP (\allowbreak {}MP-20)&\multicolumn {1}{c}{0.00 \ensuremath {\pm } 0.00}&\multicolumn {1}{c}{0.89 \ensuremath {\pm } 0.00}&\multicolumn {1}{c}{0.73 \ensuremath {\pm } 0.00}&\multicolumn {1}{c}{0.83 \ensuremath {\pm } 0.00}\\%
medicine&histology \ensuremath {\leftrightarrow } expression (\allowbreak {}HEST CCRCC)&0.37 +- 0.02&\multicolumn {1}{c}{0.62 \ensuremath {\pm } 0.00}&\multicolumn {1}{c}{0.19 \ensuremath {\pm } 0.00}&\multicolumn {1}{c}{0.67 \ensuremath {\pm } 0.00}\\%
medicine&tissue \ensuremath {\leftrightarrow } CT (\allowbreak {}CPTAC)&0.49 +- 0.01&\multicolumn {1}{c}{0.05 \ensuremath {\pm } 0.00}&\multicolumn {1}{c}{0.09 \ensuremath {\pm } 0.00}&\multicolumn {1}{c}{0.52 \ensuremath {\pm } 0.00}\\%
medicine&tissue \ensuremath {\leftrightarrow } MRI (\allowbreak {}TCGA glioma)&0.48 +- 0.02&\multicolumn {1}{c}{0.09 \ensuremath {\pm } 0.00}&\multicolumn {1}{c}{0.14 \ensuremath {\pm } 0.00}&\multicolumn {1}{c}{0.52 \ensuremath {\pm } 0.00}\\%
neuroscience&MEG \ensuremath {\leftrightarrow } text (\allowbreak {}Spanish\allowbreak {}BCBL)&0.50 +- 0.02&\multicolumn {1}{c}{0.01 \ensuremath {\pm } 0.00}&\multicolumn {1}{c}{0.07 \ensuremath {\pm } 0.00}&\multicolumn {1}{c}{0.50 \ensuremath {\pm } 0.00}\\%
neuroscience&f\allowbreak {}MRI \ensuremath {\leftrightarrow } f\allowbreak {}MRI (\allowbreak {}NSD subjects)&0.03 +- 0.03&0.62 +- 0.04&0.26 +- 0.03&0.71 +- 0.01\\\bottomrule %
\end {tabularx}

    \caption{\textbf{Complete results for all fourteen modality pairs of \cref{fig:other_domains}.}
    Modalities with a high similarity score can be aligned more easily without pairs.
    }
    \label{tab:other_domains}
\end{table}

\begin{figure}[p]
    \centering
    \begin{subfigure}[t]{\textwidth}
        \centering
        \includegraphics{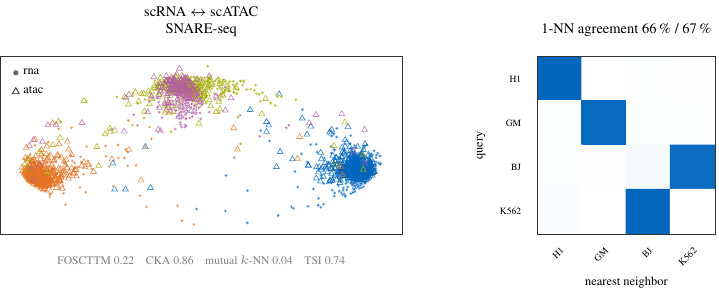} {} 
        \caption{Four cell lines, two of them swapped.}
        \label{fig:domains_snare}
    \end{subfigure}

    \vspace{0.7em}
    \begin{subfigure}[t]{\textwidth}
        \centering
        \includegraphics{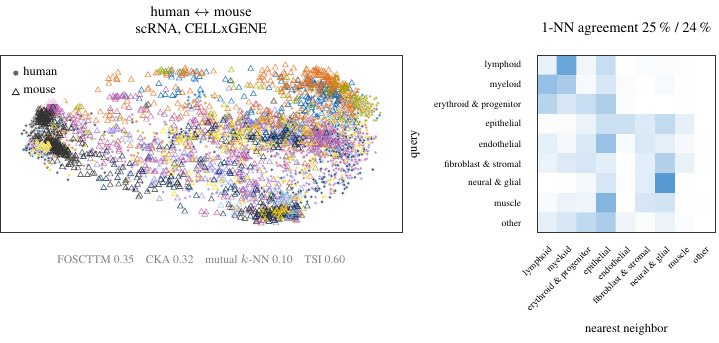} {} 
        \caption{Nine cell lineages, recovered in part.}
        \label{fig:domains_species}
    \end{subfigure}

    \vspace{0.7em}
    \begin{subfigure}[t]{\textwidth}
        \centering
        \includegraphics{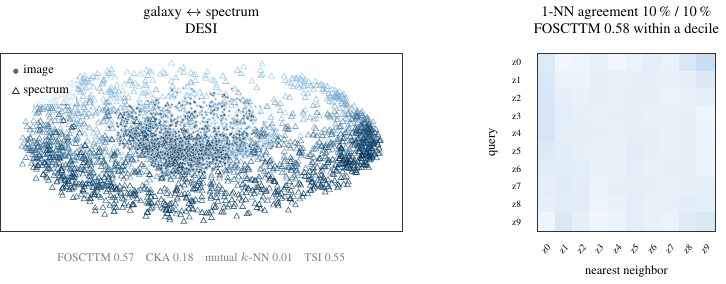} {} 
        \caption{Ten redshift deciles, not recovered at all.}
        \label{fig:domains_galaxy}
    \end{subfigure}

    \caption{\textbf{What a good, a mediocre and a failed alignment look like.} Three pairs of
    \cref{tab:other_domains}, each fitted without any pairs. We use the median seed as a representative visualization.
    Every row shows both modalities after the map, projected by one PCA fitted on their union and colored by a
    label the aligner never saw. In addition, we show the label agreement of the nearest
    cross-modal neighbor (right), the FOSCTTM and three similarity scores (below).
    (\subref{fig:domains_snare})~The four cell lines of SNARE-seq separate, but BJ and K562 are matched to each
    other, so the matrix is diagonal up to that swap.
    (\subref{fig:domains_species})~The lineages are recovered only in part. Neural cells are matched best, lymphoid
    and myeloid cells are mostly matched to each other, and muscle cells are drawn to the epithelial lineage.
    (\subref{fig:domains_galaxy})~Nothing is recovered.}
    \label{fig:domain_spectrum}
\end{figure}
\FloatBarrier
\section{Alignment with Generative Token Pooling}
\label{app:gen_token_pooling}

On MS COCO, Qwen3-8B with generative token pooling is competitive with the two embedding models and is the best of the three for DINOv2 ViT-B/14 (\cref{fig:panel_coco_train2014}). 
On the detailed captioning corpora, it is the worst of the three for every vision model.
This appendix explores the reason.
Generative token pooling asks the language model to write a description of what the text would look like and
pools the tokens it generates, rather than the tokens of the caption itself~\citep{wang2026the}.
Our hypothesis is that generation adds visual detail that a short caption leaves out but cannot add much to a long caption that already contains it.

\inparagraph{Setup.}
We compare three poolings of the same captions against the same DINOv2 ViT-B/14~\citep{Oquab:2024:DIN} image embeddings.
Generative pooling prompts Qwen3-8B~\citep{yang2025qwen3} with \emph{Imagine what it would look like to see:} followed by the caption
and averages the hidden states of at most $128$ generated tokens.
Mean pooling averages the hidden states of the caption tokens with the same backbone and no prompt, which isolates
the effect of generating.
Embedding pooling uses Qwen3-Embedding-8B~\citep{zhang2025qwen3} and serves as a reference point.
We evaluate the alignment on six datasets with varying caption lengths.
MS COCO~\citep{chen2015microsoft} at $11.9$ Qwen3 tokens on average, CC12M~\citep{changpinyo2021conceptual} at
$24.8$, WIT~\citep{srinivasan2021wit} at $41.5$, Stanford Paragraph Captioning~\citep{krause2016paragraphs} at $53.8$, DOCCI~\citep{onoe2024docci} at $140.8$ and Densely Captioned Images~\citep{urbanek2024picture} at
$167.0$.
Each score is measured on five disjoint slices of $1024$ pairs.
We report the clipped CKA of \citet{wang2026the}, which clips every feature at the mean $95$th
percentile of its absolute value and normalizes rows without centering them.
The main goal of this experiment is to isolate in which settings generative token pooling can help.
Instead of comparing the absolute level of each pooling, we compare the difference between different embedding strategies.

\inparagraph{Results.}
\cref{fig:tp_corpora} shows the difference in alignment between the three pooling methods across the six datasets.
The difference of generative and mean pooling is positive on the three corpora with short or web-scraped captions,
$+0.099$ on MS COCO, $+0.107$ on CC12M and $+0.211$ on WIT, near zero on the paragraph corpus at $+0.014$, and
negative on the two densest, $-0.144$ on DOCCI and $-0.067$ on DCI.

\inparagraph{Truncated captions.}
The problem with the above comparison is that the corpora differ in more than just caption length.
To isolate the effect of caption length, we take the two densest corpora and compare captions of decreasing detail for the same images, truncated descriptions for DOCCI and the three human-written caption variants for DCI.
For this experiment, generative and mean pooling use the smaller Qwen3-1.7B and embedding pooling uses
Qwen3-Embedding-0.6B, so its values differ from those in \cref{fig:tp_corpora} on the same captions.
\cref{fig:tp_ladder} shows the same difference in alignment as \cref{fig:tp_corpora}, but as the caption is shortened.
On DOCCI the contrast moves from $-0.108$ on the full description to $-0.057$ on half of it and to $+0.087$ on
the first sentence alone.
On DCI it moves from $-0.071$ on the full caption and $-0.059$ on the extended one to $+0.131$ on the short
human-written caption.

\begin{figure}[t]
    \centering
    \includegraphics{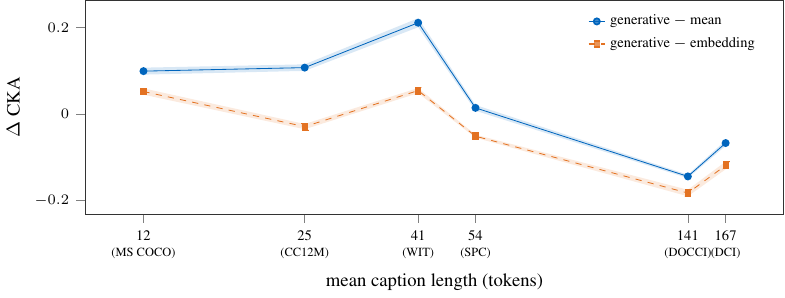} {} 
    \caption{\textbf{Generative token pooling tends to help for short captions.} 
    We compare CKA alignment for different datasets. Generative token pooling tends to be more helpful for short captions than for long, detailed ones. WIT is an outlier where generative token pooling is exceptionally effective.
    }
    \label{fig:tp_corpora}
\end{figure}

\begin{figure}[t]
    \centering
    \begin{subfigure}[b]{0.49\textwidth}
        \centering
        \includegraphics{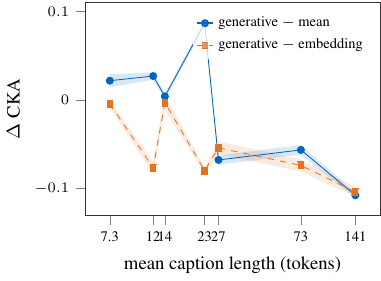} {} 
        \caption{DOCCI}
        \label{fig:tp_ladder_docci}
    \end{subfigure}%
    \hfill
    \begin{subfigure}[b]{0.49\textwidth}
        \centering
        \includegraphics{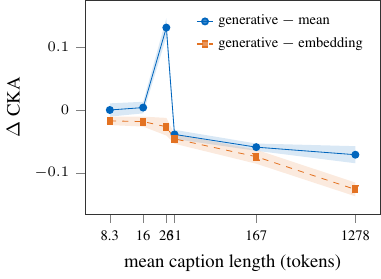} {} 
        \caption{DCI}
        \label{fig:tp_ladder_dci}
    \end{subfigure}

    \caption{\textbf{The same images, captions of decreasing detail.} 
    We compare different embedding strategies on DCI and DOCCI and find that generative token pooling can be helpful for short captions. Embedding models and normal mean pooling of the caption are generally superior for long captions.
    }
    \label{fig:tp_ladder}
\end{figure}
\FloatBarrier
\section{Granularity of Alignment}
\label{app:alignment_granularity}

Cross-modal alignment appears to depend on the scale at which it is measured.
\citet{groger2026revisiting} show that, after accounting for model depth and width, convergence is mainly expressed through local neighborhood structure rather than global distances.
However, they mainly focus on the convergence with increasing model and dataset size and not on the absolute values.
\citet{koepke2026back} show that for fixed $k$, mutual $k$-NN degrades substantially when the dataset is scaled to millions of samples, whereas the alignment is stable at a fixed ratio of $k = n/100$.
At large scales and for a fixed $k$, the alignment is much weaker than the alignment between language models, but remains above the random baseline.
Thus, their results show that fine-grained vision-language alignment is limited, while some coarser alignment remains.
Concurrent to our work, \citet{you2026converges} introduce a global counterpart to mutual $k$-NN based on minimum spanning tree (MST) edge overlap. Their analysis separates local vs. global scale from relational structure vs. metric geometry. They find models converge both in local and global relational structure, but a weaker convergence when distance agreement is required.

However, these experiments do not directly compare different levels of granularity under the same evaluation conditions.
Changing the sample count and $k$ changes the granularity of the mutual $k$-NN comparison, making the resulting alignment values difficult to compare directly across scales.
In this section, we want to test at which levels of granularity alignment exists and how strong it is when the evaluation conditions are held comparable.
We study this in three complementary experiments:

\begin{enumerate}[noitemsep, topsep=0pt, leftmargin=*]
\item A clustering-based experiment, where we can compare the alignment over cluster centers with the alignment within each cluster. We find that the alignment is substantially stronger across the cluster centers compared to inside individual clusters.
\item A PCA-based experiment, where we evaluate how much of the alignment is captured by a small number of dominant dimensions. Most of the observed alignment is already captured by the first ten PCA dimensions.
Random subspaces require substantially more dimensions to reach comparable alignment.
\item A clipping-based experiment, where one cosine similarity kernel is clipped from below or above and the alignment is evaluated after clipping.
This neglects the structure outside the clipping threshold and evaluates the impact of this structure on the general alignment. We find that cosine similarity values between $0.0$ and $0.6$ contribute most strongly to the alignment.
Very close points and far points don't share as much structure.
\end{enumerate}

\begin{figure}[t]
    \centering
    \includegraphics{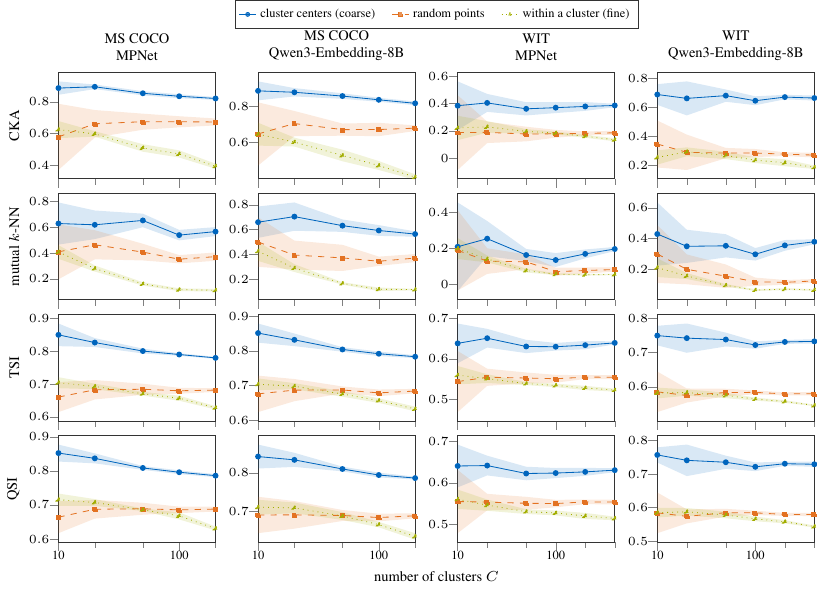} {} 
    \caption{\textbf{Alignment between cluster centers, inside clusters and on random points.} Both spaces are
    clustered jointly into $C$ clusters.
Every point of a curve scores $n = C$ points, either the $C$ cluster
    centers, $C$ members of one cluster or $C$ random samples, so the three curves differ only in how coarse the
    points are.
    The coarse cluster centers are consistently better aligned than random points or points within clusters.
    }
    \label{fig:granularity_partition}
\end{figure}

\inparagraph{Setup.}
We evaluate all experiments with DINOv2 ViT-B/14~\citep{Oquab:2024:DIN}, two language models (MPNet~\citep{Reimers:2019:SBE},  Qwen3-Embedding-8B~\citep{zhang2025qwen3}), and two datasets (MS COCO~\citep{chen2015microsoft} and WIT~\citep{srinivasan2021wit}).
We measure alignment with four alignment measures: CKA~\citep{kornblith2019similarity}, mutual $k$-NN with  $k = \max(1, \lfloor n/100 \rfloor)$~\citep{huh2024platonic}, TSI, and QSI~\citep{soares2026scalable}.
We repeat every experiment for 10 random seeds and report the mean and standard deviation.

\inparagraph{Cluster-based experiment (\cref{fig:granularity_partition}).}
We first test whether coarse groups are more strongly aligned than the fine-grained structure within those groups.
We jointly cluster the two representation spaces using size-balanced $k$-means on at most 100,000 points and lift the clustering to the full dataset for MS COCO
and to 500,000 random samples for WIT.
For each granularity, we sample the same number of points from each cluster and compare alignment across cluster centers, within clusters, and on a random subset of the same size.

Across datasets, models, and granularities, alignment across cluster centers is substantially stronger than alignment within individual clusters.
Random sampling is usually closer to the fine alignment.
This shows that the modalities agree more strongly on coarse organization than on fine-grained structure.

\begin{figure}[t]
    \centering
    \includegraphics{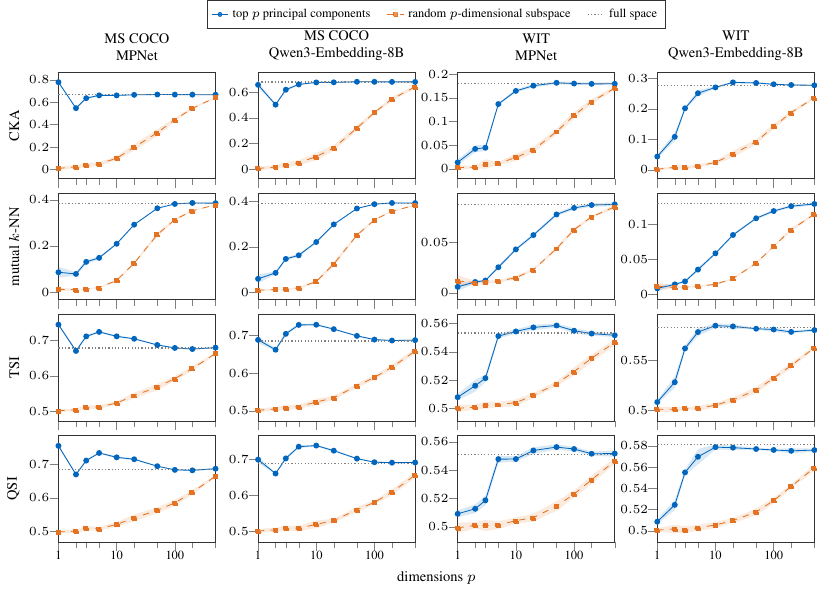} {} 
    \caption{\textbf{Alignment in the top principal components.} Each space is
    projected onto its top $p$ principal components or onto a random $p$-dimensional subspace.
    The dotted line is the alignment of the two full spaces.
    Ten principal components capture most of the alignment and sometimes surpass the alignment of the original space.
    For mutual $k$-NN, around 100 principal components are required to approximate the full alignment.
}
    \label{fig:granularity_subspace}
\end{figure}

\inparagraph{PCA-based experiment (\cref{fig:granularity_subspace}).}
We next examine whether the shared alignment is concentrated in a small number of dominant dimensions.
We subsample 10,000 points and project each representation space onto its first $p$ principal components.
We compare this to random subspaces of the same dimensionality.

The first ten PCA dimensions already capture most of the observed alignment across datasets, models, and alignment measures.
In some cases, they even produce slightly higher alignment than using more dimensions.
The only exception is mutual $k$-NN, which needs an order of magnitude more components.
Random projections require substantially more dimensions to reach comparable alignment.

\begin{figure}[t]
    \centering
    \includegraphics{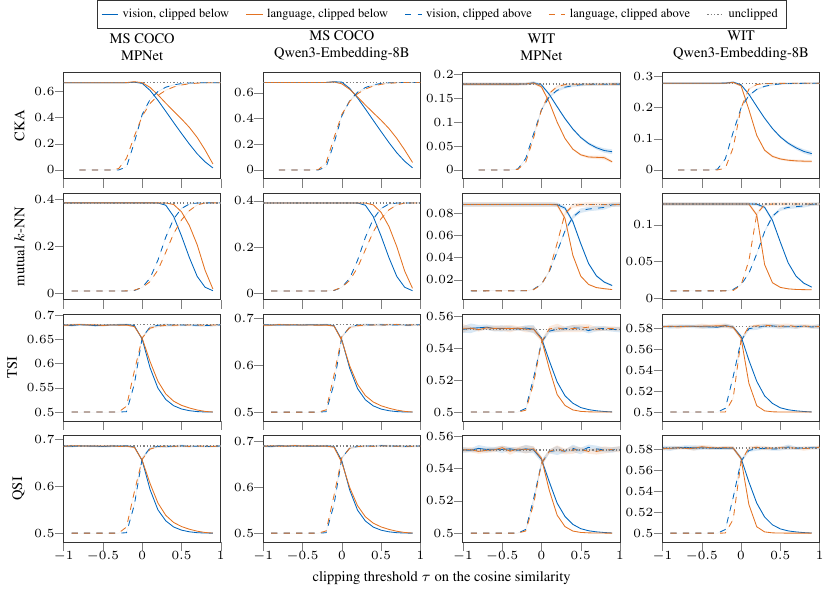} {} 
    \caption{\textbf{Alignment after clipping one similarity kernel.} The centered cosine-similarity kernel of one space is clipped from below or from above at the threshold $\tau$, while the kernel of the other space is left unchanged.
    Points with a cosine similarity of $0.6$ or higher and points farther away than orthogonal contribute only marginally to the alignment score.
    }
    \label{fig:granularity_clipping}
\end{figure}

\inparagraph{Clipping-based experiment (\cref{fig:granularity_clipping}).}
Finally, we test which ranges of cosine similarity contribute most to the alignment.
We clip a cosine-similarity kernel from below or above and measure alignment after clipping on subsets of $10{,}000$ points.

We find that cosine similarities between $0.0$ and $0.6$ carry most of the alignment. 
The structure outside this range is largely redundant.
Clipping similarities above $0.6$ has only a limited effect, indicating that the precise structure of very close points is not strongly shared.
Similarly, moving all negative cosine similarities to zero has little effect on CKA and mutual $k$-NN, suggesting that the precise structure of very distant points contributes little.

Overall, the strongest shared structure lies at an intermediate, coarse similarity scale.
The modalities share coarse relationships between samples, and the structure of very close and very distant points adds little beyond them.

\end{document}